\def\year{2020}\relax
\documentclass[letterpaper]{article} 
\usepackage{aaai20}  
\usepackage{times}  
\usepackage{helvet} 
\usepackage{courier}  
\usepackage[hyphens]{url}  
\usepackage{graphicx} 
\usepackage{graphicx}  
\usepackage[protrusion=true,expansion=true]{microtype}		

\usepackage{amsmath}
\usepackage{amsthm}
\usepackage{dsfont}
\usepackage{amssymb}
\usepackage{mathtools}
\usepackage{mathrsfs}
\mathtoolsset{showonlyrefs}

\newtheorem{lemma}{Lemma}[]
\newtheorem{prop}{Property}[]

\newtheorem{assumpA}{Assumption}

\usepackage{courier} 

\usepackage{lipsum} 

\usepackage{subcaption}

\usepackage{xcolor}

\DeclareMathOperator*{\argmax}{arg\,max}

\usepackage{algpseudocode}
\usepackage[vlined,linesnumbered,ruled,algo2e]{algorithm2e}
\SetKwProg{Fn}{Function}{}{}
\let\oldnl\nl
\newcommand{\nonl}{\renewcommand{\nl}{\let\nl\oldnl}}

\newcommand{\citet}[1]{\citeauthor{#1} \shortcite{#1}}
\newcommand{\citep}{\cite}

\title{Reinforcement Learning When All Actions are Not Always Available}
\author{
Yash Chandak\textsuperscript{\rm 1} \,\,\, Georgios Theocharous\textsuperscript{\rm 2} \,\,\, Blossom Metevier\textsuperscript{\rm 1} \,\,\, Philip S. Thomas\textsuperscript{\rm 1} \\
\textsuperscript{\rm 1}University of Massachusetts Amherst, \textsuperscript{\rm 2}Adobe Research \\
\{ychandak,bmetevier,pthomas\}@cs.umass.edu \,\,\,\, theochar@adobe.com
}
 
\begin{document}

\maketitle

\begin{abstract}
The Markov decision process (MDP) formulation used to model many real-world sequential decision making problems does not efficiently capture the setting where the set of available decisions (actions) at each time step is stochastic. 
Recently, the stochastic action set Markov decision process (SAS-MDP) formulation has been proposed, which better captures the concept of a stochastic action set. 
In this paper we argue that existing RL algorithms for SAS-MDPs can suffer from potential divergence issues, and present new policy gradient algorithms for SAS-MDPs that incorporate variance reduction techniques unique to this setting, and provide conditions for their convergence. 
We conclude with experiments that demonstrate the practicality of our approaches on tasks inspired by real-life use cases wherein the action set is stochastic.
\end{abstract}

\section{Introduction}

In many real-world sequential decision making problems, the set of available decisions, which we call the \emph{action set}, is stochastic. 
In vehicular routing on a road network \citep{gendreau1996stochastic} or packet routing on the internet \citep{ribeiro2008optimal}, the goal is to find the shortest path between a source and destination. 
However, due to construction, traffic, or other damage to the network, not all pathways are always available. 
In online advertising \citep{tan2012online,mahdian2007allocating}, the set of available ads can vary due to fluctuations in advertising budgets and promotions. 
In robotics \citep{feng2000optimal}, actuators can fail.
In recommender systems \citep{harper2007user}, the set of possible recommendations can vary based on product availability.
These examples capture the broad idea and motivate the question we aim to address: \textit{how can we develop efficient learning algorithms for sequential decision making problems wherein the action set can be stochastic?}

Sequential decision making problems \emph{without} stochastic action sets are typically modeled as \emph{Markov decision processes} (MDPs). 
Although the MDP formulation is remarkably flexible, and can incorporate concepts like stochastic state transitions, partial observability, and even different (deterministic) action availability depending on the state, it cannot efficiently incorporate stochastic action sets. 
As a result, algorithms designed for MDPs are not well suited to our setting of interest. 
Recently, \citet{Boutilier2018PlanningAL} laid the foundations for \emph{stochastic action set Markov decision processes} (SAS-MDPs), that extends MDPs to include stochastic action sets. 
They also showed how the Q-learning and value iteration algorithms, two classic algorithms for approximating optimal solutions to MDPs, can be extended to SAS-MDPs.

In this paper we show that the lack of convergence guarantees of the Q-learning algorithm, when using function approximators in the MDP setting can potentially get exacerbated in the SAS-MDP setting. 
We therefore derive policy gradient and natural policy gradient algorithms for the SAS-MDP setting and provide conditions for their almost-sure convergence. 
Critically, since the introduction of stochastic action sets introduces further uncertainty in the decision making process, variance reduction techniques are of increased importance. 
We therefore derive new approaches to variance reduction for policy gradient algorithms that are unique to the SAS-MDP setting. 
We validate our new algorithms empirically on tasks inspired by real-world problems with stochastic action sets.

\section{Related Work}

While there is extensive literature on solving sequential decision problems modeled as MDPs \citep{sutton2018reinforcement}, there are few methods designed to handle stochastic action sets.  
Recently, \citet{Boutilier2018PlanningAL} laid the foundation for studying MDPs with stochastic action sets by defining the new SAS-MDP problem formulation, which we review in the background section. 
After defining SAS-MDPs, \citet{Boutilier2018PlanningAL} presented and analyzed the model-based value iteration and policy iteration algorithms and the model-free Q-learning algorithm for SAS-MDPs. 

In the bandit setting, wherein individual decisions are optimized rather than sequences of dependent decisions, \textit{sleeping bandits} extend the standard bandit problem formulation to allow for stochastic action sets \citep{kanade2009sleeping,kleinberg2010regret}. 
We focus on the SAS-MDP formulation rather than the sleeping bandit formulation because we are interested in sequential problems. 
Such sequential problems are more challenging because making optimal decisions requires one to reason about the long-term impact of decisions, which includes reasoning about how a decision will influence the probability that different actions (decisions) will be available in the future.

Although we focus on the \emph{model-free} setting, wherein the dynamics of the environment are not known \emph{a priori} to the agent optimizing its decisions, in the alternative \emph{model-based} setting researchers have considered related problems in the area of \emph{stochastic routing} \citep{papadimitriou1991shortest,polychronopoulos1996stochastic,nikolova2006optimal,nikolova2008route}. 
In stochastic routing problems, the goal is to find a shortest path on a graph with stochastic availability of edges. 
The SAS-MDP framework generalizes stochastic routing problems by allowing for sequential decision making problems that are not limited to shortest path problems. 

\section{Background}
\label{sec:background}
MDPs and SAS-MDPs \citep{Boutilier2018PlanningAL} are mathematical formulations of sequential decision problems. 
Before defining SAS-MDPs, we define MDPs. 
We refer to the entity interacting with an MDP or SAS-MDP and trying to optimize its decisions as the \emph{agent}.

Formally, an MDP is a tuple $\mathcal{M} = (\mathcal{S},\mathcal{B},\mathcal{P},\mathcal{R}, \gamma, d_0)$.
$\mathcal{S}$ is the set of all possible states that the agent can be in, called the \emph{state set}. 
Although our math notation assumes that $\mathcal S$ is countable, our primary results extend to MDPs with continuous states. 
$\mathcal{B}$ is a finite set of all possible actions that the agent can take, called the \emph{base action set}. 
$S_t$ and $A_t$ are random variables that denote the state of the environment and action chosen by the agent at time $t \in \{0,1,\dotsc\}$. 
$\mathcal P$ is called the \emph{transition function} and characterizes how states transition: $\mathcal P(s,a,s')\coloneqq \Pr(S_{t+1}=s'|S_t=s,A_t=a)$. 
$R_t \in [-R_\text{max}, R_\text{max}]$, a bounded random variable, is the scalar reward received by the agent at time $t$, where $R_\text{max}$ is a finite constant. 
$\mathcal R$ is called the \emph{reward function}, and is defined as $\mathcal R(s,a)\coloneqq\mathbb{E}[R_t|S_t=s,A_t=a]$. 
%
The reward discount parameter, $\gamma \in [0,1)$, characterizes how to utility of rewards to the agent decays based on how far in the future they occur. 
We call $d_0$ the \emph{start state distribution}, which is defined as $d_0(s)\coloneqq\Pr(S_0=s)$.

We now turn to defining a SAS-MDP. 
Let the set of actions available at time $t$ be a random variable, $\mathcal A_t \subseteq \mathcal B$, which we assume is always not empty, i.e., $\mathcal A_t\neq \emptyset$.
Let $\varphi$ characterize the conditional distribution of $\mathcal A_t$: 
$\varphi(s, \alpha) \coloneqq \Pr(\mathcal A_t = \alpha | S_t = s)$.
We assume that $\mathcal A_t$ is Markovian, in that its distribution is conditionally independent of all events prior to the agent entering state $S_t$ given $S_t$. 
Formally, a SAS-MDP is $\mathcal M' = \{\mathcal M \cup \varphi \}$, 
with the additional requirement that $A_t \in \mathcal A_t$. 

A policy $\pi:\mathcal S \times 2^{\mathcal  B} \times \mathcal B \to [0,1]$ is a conditional distribution over actions for each state: $\pi(s, \alpha, a)\coloneqq \Pr(A_t=a|S_t=s, \mathcal A_t = \alpha)$ for all $s \in \mathcal S, a \in \alpha, \alpha \subseteq \mathcal B$, and $t$, where $\alpha \neq \emptyset$.
Sometimes a policy is parameterized by a weight vector $\theta$, such that changing $\theta$ changes the policy. 
We write $\pi^\theta$ to denote such a parameterized policy with weight vector $\theta$. 
%
%
For any policy $\pi$, we define the corresponding \emph{state-action value function} to be $q^\pi(s,a) \coloneqq \mathbb{E}[\sum_{k=0}^{\infty}\gamma^k R_{t+k} |S_t=s, A_t=a, \pi]$, where conditioning on $\pi$ denotes that $A_{t+k} \sim \pi(S_{t+k}, \mathcal A_{t+k}, \cdot)$ for all $\mathcal A_{t+k}$ and $S_{t+k}$ for $k \in [t+1, \infty)$. 
Similarly, the \emph{state-value function} associated with policy $\pi$ is $v^\pi(s)\coloneqq \mathbb{E}[\sum_{k=0}^{\infty}\gamma^k R_{t+k} |S_t=s, \pi]$.
For a given SAS-MDP $\mathcal{M}'$, the agent's goal is to find an \emph{optimal policy}, $\pi^*$, (or equivalently optimal policy parameters $\theta^*$) which is any policy that maximizes the expected sum of discounted future rewards.
More formally, an optimal policy is any $\pi^{*} \in \text{argmax}_{\pi \in \Pi} J(\pi)$, where $J(\pi)\coloneqq \mathbb{E} [\sum_{t=0}^{\infty}\gamma^tR_t |\pi]$ and $\Pi$ denotes the set of all possible policies.
For notational convenience, we sometimes use $\theta$ in place of $\pi$, e.g., to write $v^\theta$, $q^\theta$, or $J(\theta)$, since a weight vector $\theta$ induces a specific policy.

%
As shown by \citet{Boutilier2018PlanningAL}, one way to model stochastic action sets using the \textit{MDP formulation (rather than the SAS-MDP formulation)} is to 
define states such that one can infer $\mathcal A_t$ from $S_t$. 
%
Transforming an MDP into a new MDP with $\mathcal A_t$ embedded in $S_t$ in this way 
can result in the size of the state set growing exponentially---
by a factor of $2^{|\mathcal{B}|}$.
This drastic increase in the size of the state set can make finding or approximating an optimal policy prohibitively difficult. 
%
Using the SAS-MDP formulation, the challenges associated with this exponential increase in the size of the state set can be avoided, and one can derive algorithms for finding or approximating optimal policies in terms of the state set of the original underlying MDP. 
This is accomplished using a variant of the 
\textit{Bellman operator}, $\mathcal T$, which incorporates the concept of stochastic action sets:
\begin{align}
    \mathcal{T}^{\pi}v(s) =& \sum_{\alpha \in 2^\mathcal B} \varphi(s,\alpha)\sum_{a \in \alpha}\pi(s,\alpha,a)\Big(  \sum_{s' \in \mathcal S} P(s,a,s')
    \\
    &~~~~~~~~~~~~~~~~~~~~~~~~~~~~~~~~~~~(R(s,a) + \gamma v(s')) \Big)
    \label{eqn:SAS-bellman}
\end{align}
for all $s \in \mathcal S$. 
Similarly, one can extend the \emph{Bellman optimality operator} \citep{sutton2018reinforcement}:
\begin{align}
    \mathcal{T}^*\!v(s) =&\!\!\!\sum_{\alpha \in 2^\mathcal B}\!\! \varphi(s,\alpha)\max_{a \in \alpha}\!\! \sum_{s' \in \mathcal S} \!\!P(s,a,s')(R(s,a) + \gamma v(s')).
    %
\end{align}
\citet{Boutilier2018PlanningAL} showed that the stationary optimal policies exists for SAS-MDPs and can be represented using (state-specific) decision lists (or orderings/rankings) over the action set.
As a policy takes into account the available set of actions, an optimal policy chooses the highest ranked action from those that are available.
%
Building upon these results, \citet{Boutilier2018PlanningAL} proposed the following update for a tabular estimate, $q$, of $q^{\pi^*}$:
\begin{align}
    q(S_t, A_t) \leftarrow (1 - \eta)q (S_t, A_t) + \eta (R_t + \gamma \!\!\!\underset{a \in \mathcal A_{t+1}}{\text{max}}\!\!\!q(S_{t+1}, a)). \label{eqn:SAS-q}
\end{align}
Notice that the maximum is computed only over the available actions, $\mathcal A_{t+1}$, in state $S_{t+1}$.
%
%
We refer to the algorithm using this update rule as \emph{SAS-Q-learning}.
\section{Potential Limitations of SAS-Q-Learning}
\label{sec:limitations}
%

%
Although SAS-Q-learning provides a powerful first model-free algorithm for approximating optimal policies for SAS-MDPs, it inherits several of the drawbacks of the Q-learning algorithm for MDPs.  
%
%
Just like Q-learning, in a state $S_t$ and with available actions $\mathcal A_t$, the SAS-Q-learning method chooses actions deterministically when not exploring: $A_t\in \argmax_{a \in \mathcal A_t}q(S_t,a)$. 
This limits its practicality for problems where optimal policies are stochastic, which is often the case when the environment is partially observable or when the use of function approximation causes state aliasing \citep{baird1995residual}. 
Additionally, if the SAS-Q-learning update converges to an estimate, $q$, of $q^{\pi^*}$ such that $\mathcal T v(s)=v(s)$ for all $s \in \mathcal S$, 
then the agent will act optimally; 
however, convergence to a fixed-point of $\mathcal T$ is seldom achieved in practice, and reducing the difference between $v(s)$ and $\mathcal T v(s)$ (what SAS-Q-learning aims to do) 
does not ensure improvement of the policy \citep{sutton2018reinforcement}.

SAS-Q-learning does not perform gradient ascent or descent on \emph{any} function, and it can cause divergence of the estimator $q$ when using function approximation, just like Q-learning for MDPs \citep{baird1995residual}.
In the setting where all actions are always available, SAS-Q-learning reduces to standard Q-learning. Therefore, for all the cases in this setting where Q-learning is unstable, SAS-Q-learning is also unstable.
In the setting where all actions are \textit{not} always available, there exist additional cases where Q-learning is stable but SAS-Q-learning is not.
However, in such cases where Q-learning is stable, its solution might not be particularly useful as it does not incorporate the notion of stochasticity in the action set (Section 8, Fig.2, \citeauthor{Boutilier2018PlanningAL} 2018). 
%
%

\begin{figure}
  \begin{center}
    \includegraphics[width=0.18\textwidth]{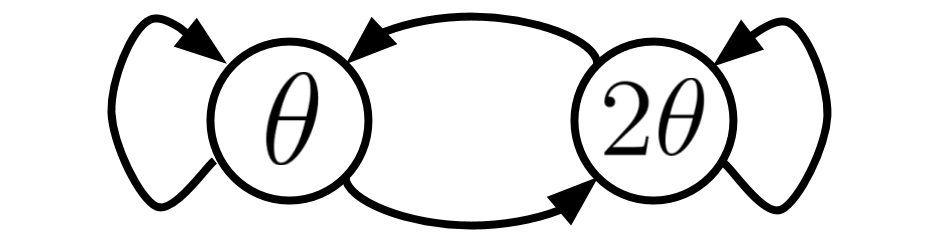}
  \end{center}
  \caption{$\theta \rightarrow 2\theta$ MDP}		\label{Fig:SAS_q_divergence}
\end{figure}
To see this, consider the SAS variant of the classical $\theta \rightarrow 2\theta$ MDP \citep{tsitsiklis1983analysis} illustrated in Figure \ref{Fig:SAS_q_divergence}. 
In this example there are two states, $s_1$ (left in Figure \ref{Fig:SAS_q_divergence}) and $s_2$ (right), and two actions, $a_1=\text{left}$ and $a_2=\text{right}$. 
The agent in this example uses function approximation \citep{sutton2018reinforcement}, with weight vector $\theta \in \mathbb R^2$, such that $q(s_1,a_1)=\theta_1, q(s_2,a_1)=2\theta_1$ and $q(s_1,a_2)=\theta_2, q(s_2,a_2)=2\theta_2$. 
In either state, if the agent takes the left action, it goes to the left state, and if the agent takes the right action, it goes to the right state. 
In our SAS-MDP version of this problem, both actions are not always available. 
Let $R_t=0$ always, and $\gamma = 1$. 
Consider the case where the weights of the $q$-approximation are initialized to
$\theta=\lbrack -2, -5 \rbrack$. 
Now suppose that a transition is observed from the left state to the right state, and after the transition the left action is \textit{not} available to the agent. 
As per the SAS-Q-learning update rule provided in \eqref{eqn:SAS-q},
$\theta_2 \leftarrow \theta_2 + \eta (r + \gamma 2\theta_2 - \theta_2).$
Since $r=0$ and $\gamma =1$, this is equivalent to 
$\theta_2 \leftarrow \theta_2 + \eta \theta_2.$
Considering the off-policy setting where this transition is used repeatedly on its own, then irrespective of the learning rate, 
$\eta > 0$, the weight $\theta$ would diverge to $-\infty$.
In contrast, had there been no constraint of using max over $q$ given the available actions, the Q-learning update would have been,
$\theta_2 \leftarrow \theta_2 + \eta (r + \gamma 2 \theta_1 - \theta_2)$ because action $a_1$ has higher q-value than $a_2$ due to $\theta_1 > \theta_2$.
This would make $\theta_2$ converge to the value $-4$ (the correct answer is $0$).
%
%

This provides an example of how the stochastic constraints on the set of available actions can be instrumental in causing the SAS-Q-learning method to diverge, and ignoring the stochastic constraint can prevent Q-learning from converging to the correct solution. 
We suspect more such cases can be constructed by adapting examples from non-SAS setup ( \citeauthor{baird1995residual} 1995,  \citeauthor{Gordon96chatteringin} 1996, Chpt 11.2 \citeauthor{sutton2018reinforcement} 2018). 

%
\section{Policy Gradient Methods for SAS-MDPs}
\label{sec:PG}
In this section we provide an alternative to the SAS-Q-learning algorithm by deriving policy gradient algorithms \citep{sutton2000policy} for the SAS-MDP setting. 
While the Q-learning algorithm minimizes the error between $\mathcal Tv(s)$ and $v(s)$ for all states $s$ (using a procedure that is not a gradient algorithm), policy gradient algorithms perform stochastic gradient ascent on the objective function $J$. 
That is, they use the update $\theta \gets \theta + \eta \Delta$, where $\Delta$ is an unbiased estimator of $\nabla J(\theta)$.

Unlike the Q-learning algorithm, policy gradient algorithms for MDPs provide convergence guarantees to a critical point (local/global optima) even when using function approximation, and can approximate optimal stochastic policies. 
However, ignoring the fact that actions are not always available and using off-the-shelf algorithms for MDPs fails to fully capture the problem setting 
\citep{Boutilier2018PlanningAL}.
It is therefore important that we derive policy gradient algorithms that are appropriate for the SAS-MDP setting, as they provide the first convergent model-free algorithms for SAS-MDPs when using function approximation. 
In the following lemma we extend the expression for the policy gradient for MDPs \citep{sutton2000policy,thomas2014bias} to handle stochastic action sets. 
\begin{lemma}[SAS Policy Gradient]
\label{lemma:SASPG}
For a SAS-MDP, for all $s \in \mathcal S$,
    \begin{align}
       \nabla J(\theta) =& \sum_{t=0}^{\infty} \sum_{s \in \mathcal S}\gamma^t \Pr(S_{t}=s|\theta)\Big ( \sum_{\alpha \in 2^{\mathcal{B}}} \varphi(s, \alpha)
       \\ 
       & ~~~~~~~~~~~ \sum_{a \in \alpha} q^\theta(s,a)\frac{\partial \pi^\theta(s,\alpha, a)}{\partial \theta}\Big). \label{eqn:SAS_pg}
    \end{align}
\end{lemma}
\begin{proof}
 See Appendix A.
\end{proof}
It follows from 
Lemma \ref{lemma:SASPG} that we can create unbiased estimates of $\nabla J(\theta)$, which can be used to update $\theta$ using the well-known stochastic gradient ascent algorithm. 
This algorithm is presented in Algorithm \ref{alg1}. 
Notably, this process does \emph{not} require the agent to know $\varphi$.
Also, similar to the SAS-Q-learning method, the policy can be parameterized such that it is not required to embed the available actions as a part of the state.
One such parameterization is provided in Appendix F. 
Notice that in the special case where all actions are always available, the expression in Lemma \ref{lemma:SASPG} degenerates to the policy gradient theorem for MDPs \citep{sutton2018reinforcement}.
We now establish that SAS policy gradient algorithms are guaranteed to converge to locally optimal policies under the following standard assumptions, 
%
%
\begin{assumpA}[Differentiable]
For any state, action-set, and action triplet $(s, \alpha, a)$, policy $\pi^\theta(s, \alpha, a)$ is continuously differentiable in the parameter $\theta$.
\label{apx:ass:1}
\end{assumpA}
\begin{assumpA}[Lipschitz smooth gradient]
Let $\Theta$ denote the set of all possible parameters for policy $\pi^\theta$, then for some constant $L$,
$$\lVert \nabla J(\theta) - \nabla J(\bar \theta)\rVert \leq L \lVert \theta - \bar \theta \rVert  \hspace{10pt} \forall \theta, \bar \theta \in \Theta.$$
\label{apx:ass:2}
\end{assumpA}
\begin{assumpA}[Learning rate schedule]
Let $\eta_\theta^t$ be the learning rate for updating policy parameters $\theta$, then,
$$\sum_{t=0}^{\infty} \eta_\theta^t = \infty,
~~~~ \sum_{t=0}^{\infty} (\eta_\theta^t)^2 < \infty.$$
\label{apx:ass:3}
\end{assumpA}

All the assumptions (\ref{apx:ass:1}-\ref{apx:ass:3}) are satisfied under standard policy parameterization techniques (linear-function/neural-networks with softmax) and appropriately set learning rates.

\begin{lemma}
Under Assumptions \eqref{apx:ass:1}-\eqref{apx:ass:3}, the SAS policy gradient algorithm causes $\nabla J(\theta_t) \to 0$ as $t \to \infty$, with probability one.
\end{lemma}
\begin{proof}
 See Appendix B.
\end{proof}

\emph{Natural policy gradient} algorithms \citep{kakade2002natural} extend policy gradient algorithms to follow the \emph{natural gradient} of $J$ \citep{amari1998natural}. 
In essence, whereas policy gradient methods perform gradient ascent in the space of policy parameters by computing the gradient of $J$ as a function of the parameters $\theta$, natural policy gradient methods perform gradient ascent in the space of policies (which are probability distributions) by computing the gradient of $J$ as a function of the policy, $\pi$. 
Thus, whereas policy gradient implicitly measures distances between policies by the Euclidean distance between their policy parameters, natural policy gradient methods measure distances between policies using notions of distance between probability distributions. 
In the most common form of natural policy gradients, the distances between policies are measured using a Taylor approximation of \textit{Kullback}-\textit{Leibler divergence} (KLD). 
By performing gradient ascent in the space of policies rather than the space of policy parameters, the natural policy gradient becomes invariant to how the policy is parameterized \citep{thomas2018decoupling}, which can help to mitigate the vanishing gradient problem in neural networks and improve learning speed \citep{amari1998natural}.

The natural policy gradient (using a Taylor approximation of KLD to measure distances) is $\widetilde \nabla J(\theta)\coloneqq F_\theta^{-1}\nabla J(\theta)$ where $F_\theta$ is the \emph{Fisher information matrix} (FIM) associated with the policy $\pi_\theta$. 
Although the FIM is a well-known quantity, it is typically associated with a parameterized probability distribution. 
Here, $\pi_\theta$ is a \emph{collection} of probability distributions---one per state. 
This raises the question of what $F_\theta$ should be when computing the natural policy gradient. 
Following the work of \citet{bagnell2003covariant} for MDPs, we show that the FIM, $F_\theta$, for computing the natural policy gradient for a SAS-MDP can also be derived by viewing $\pi_\theta$ as a  distribution over possible \emph{trajectories} (sequences of states, available action sets and executed actions). 
\begin{prop}[Fisher Information Matrix]
For a policy, parameterized using weights $\theta$, let $\psi^\theta(s, \alpha, a) \coloneqq$ $ \partial \log \pi^\theta (s, \alpha, a)/\partial \theta$, then the Fisher information matrix is,
\begin{align}
    \footnotesize
    F_\theta =& \sum_{t=0}^{\infty} \sum_{s \in \mathcal S}\gamma^t \Pr(S_{t}=s|\theta)\!\!\sum_{\alpha \in 2^{\mathcal{B}}} \!\!\! \Big (\varphi(s, \alpha)
    \\
    &~~~~~~~~~~~~ \sum_{a \in \alpha}\pi^\theta(s,\alpha, a)\psi^\theta(s, \alpha, a) \psi^\theta(s, \alpha, a)^\top \Big).
\end{align}
\end{prop}
\begin{proof}
See Appendix C.
\end{proof}

Furthermore, \citet{kakade2002natural} showed that many terms in the definition of the natural policy gradient cancel, providing a simple expression for the natural gradient 
%
which can be estimated with time linear in the number of policy parameters per time step. 
We extend the result of \citet{kakade2002natural} to the SAS-MDP formulation in the following lemma:
\begin{lemma}[SAS Natural Policy Gradient]
\label{lemma:sasnpg}
Let $w$ be a parameter such that,
\begin{align}
   \frac{\partial }{\partial w} \mathbb{E} \left [\frac{1}{2} \sum_t^\infty \gamma^t \left(\psi^\theta(S_t, \mathcal A_t, A_t)^\top w - q^\theta(S_t, A_t)\right)^2 \right] = 0, \label{eqn:fisher_grad}
\end{align}
then for all $s \in \mathcal S$ in $\mathcal M'$,  $\widetilde \nabla J(\theta) = w.$
\end{lemma}
\begin{proof}
See Appendix C.
\end{proof}
From Lemma \ref{lemma:sasnpg}, we can derive a computationally efficient natural policy gradient algorithm by using the well-known \emph{temporal difference} algorithm \citep{sutton2018reinforcement}, modified to work with SAS-MDPs, to estimate $q^\theta$ with the approximator $\psi^\theta(S_t,\mathcal A_t, A_t)^\top w $, and then using the update $\theta \gets \theta + \eta w$. 
This algorithm, which is the SAS-MDP equivalent of NAC-TD \citep{bhatnagar2008incremental,Degris2012,Morimura2005,Thomas2012}, is provided in Algorithm 2 in Appendix E.

\section{Adaptive Variance Mitigation}
%
\label{sec:var}
In the previous section, we derived (natural) policy gradient algorithms for SAS-MDPs. 
While these algorithms avoid the divergence of SAS-Q-learning, they suffer from the high variance of policy gradient estimates \citep{kakade2003sample}.
As a consequence of the additional stochasticity that results from stochastic action sets, 
this problem can be even more severe in the SAS-MDP setting. 
In this section, we 
leverage insights from the Bellman equation for SAS-MDPs, provided in \eqref{eqn:SAS-bellman}, to reduce the variance of policy gradient estimates. 

One of the most popular methods to reduce variance is the use of a state-dependent baseline $b(s)$. 
\citet{sutton2000policy} showed that, for any state-dependent baseline $b(s)$:
\begin{align}
    \nabla J(\theta) = \mathbb{E}\left[\sum_{t=0}^\infty \gamma^t \psi^\theta(s,\alpha, a)\middle(q^\theta(s,a) - b(s) \middle)\right]. \label{eqn:baseline}
\end{align}
For any random variables $X$ and $Y$, we know that the variance of $X - Y$ is given by $\text{var}(X-Y) = \text{var}(X) + \text{var}(Y) - 2 \text{cov}(X, Y)$, where cov stands for covariance.
Therefore, the variance of $X - Y$ is lesser than variance of $X$ if $2 \text{cov}(X, Y) >  \text{var}(Y)$.
As a result, any state dependent baseline $b(s)$  whose value is sufficiently correlated to the expected return, $q^\theta(s, a)$, can be used to reduce the variance of the sample estimator of \eqref{eqn:baseline}.
A baseline dependent on both the state and action can have higher correlation with $q^\theta(s,a)$, and could therefore reduce variance further. 
However, such action dependent baselines cannot be used directly, as they can result in biased gradient estimates. 
Developing such baselines remains an active area of research for MDPs \citep{thomas2017policy,grathwohl2017backpropagation,liu2017action,wu2018variance,tucker2018mirage} and is largely complementary to our purpose.
Further, even the optimal state-dependent baseline \cite{greensmith2004variance}, which leads to the minimum variance gradient estimator, is not feasible to compute and only under certain restrictive assumptions reduces to the common choice of state-value function estimator, $\hat v(s)$.
Therefore, in the following, we propose multiple baselines that  are easy to compute, and then combine them optimally.

We now introduce a baseline for SAS-MDPs that lies between state-dependent and state-action-dependent baselines. 
Like state-dependent baselines, these new baselines do not introduce bias into gradient estimates. 
However, like action-dependent baselines these new baselines include \emph{some} information about the chosen actions. 
Specifically, we propose baselines that depend on the state, $S_t$, and available action set $\mathcal A_t$, but not the precise action, $A_t$. 

Recall from the SAS Bellman equation \eqref{eqn:SAS-bellman} that the state-value function for SAS-MDPs can be written as,
$ v^\theta(s) =\sum_{\alpha \in 2^{\mathcal B}}\varphi(s, \alpha)\sum_{a \in \alpha}  \pi^\theta(s, \alpha, a) q^\theta(s, a)$.
While we cannot directly use a baseline dependent on the action sampled from $\pi^\theta$, we \textit{can} use baseline dependent on the sampled action \textit{set}. 
We consider a new baseline which leverages this information about the sampled action set $\alpha$. 
This baseline is $
    \bar q(s, \alpha) \coloneqq \sum_{a \in \alpha} \pi^\theta(s, \alpha, a) \hat q(s, a),
$
where $\hat q$ is a learned estimator of the state-action value function, and $\bar q$ represents its expected value under the current policy, $\pi^\theta$, conditioned on the sampled action set $\alpha$.

In principle, we expect $\bar q(S_t,\mathcal A_t)$ to be more correlated with $q^\theta(S_t,A_t)$ as it explicitly conditions on the action set and does not compute an average over all action sets possible, like $\hat v$. 
Practically, however, estimating $q$ values can be harder than estimating $v$.
This can be attributed to the fact that with the same number of training samples, the number of parameters to learn in $\hat q$ is more than those in an estimate of $v^\theta$. 
This poses a new dilemma of deciding when to use which baseline. 
To get the best of both, we consider using a weighted combination of $\hat v(S_t)$ 
and $\bar q(S_t,\mathcal A_t)$. 
In the following property we establish that using any weighted combination of these two baselines results in an unbiased estimate of the SAS policy gradient. 

\IncMargin{1em}
	\begin{algorithm2e}[t]
		$\mathbf{A} = \lbrack\lambda_1,\lambda_2 \rbrack^\top = \lbrack -0.5, -0.5 \rbrack^\top$
		\Comment{Initialize $\lambda$'s}\\
		\For {$episode = 0,1,2...$}{
			\nonl \textcolor[rgb]{0.5,0.5,0.5}{\# Collect transition batch  using $\pi^\theta$}
			\\
    		$\mathds{B} = \{(s_0, \alpha_0, a_0, r_0), ..., (s_T, \alpha_T, a_T, r_T)  \}$  
    		\\
    		$\hat G(s_t) = \sum_{k=0}^{T -t} \gamma^k r_{t+k}  $
    		\\
    		\vspace{8pt}
    		\nonl \textcolor[rgb]{0.5,0.5,0.5}{\# Perform update on parameters using batch $\mathbb{B}$}
    		\\
    		$\psi^\theta (s, \alpha, a) = \frac{ \partial \log \pi^\theta(s, \alpha, a)}{\partial \theta}$
    		\\
    		$\varpi \leftarrow \varpi + \eta_\varpi (\hat G(s) - \hat v^\varpi(s))\frac{\partial \hat v^\varpi(s)}{\partial \varpi} $  	
    		\\
    		$\omega \leftarrow \omega + \eta_\omega (\hat G(s) -  \bar q^{\omega} (s, \alpha))\frac{\partial \bar q^{\omega} (s, \alpha)}{\partial \omega}$  
    		\\
    		$\theta \leftarrow \theta + \eta_\theta (\hat G(s) +\lambda_1\hat v^\varpi(s) + \lambda_2  \bar q^{\omega} (s, \alpha)) \psi^\theta (s, \alpha, a)$ \Comment{ Update $\pi^{\theta}$}
    		\\
    		\vspace{8pt}
    		\nonl \textcolor[rgb]{0.5,0.5,0.5}{\# Automatically tune hyper-parameters for variance reduction using $\mathbb{B}$}
    		\\
    		$\mathbf{B} = \lbrack  \psi^\theta(s,\alpha, a) \hat v^\varpi(s) , \psi^\theta(s,\alpha, a) \bar q^\omega(s, \alpha)\rbrack^\top$ 
            \\
            $\mathbf{C} = \lbrack \psi^\theta(s,\alpha, a) \hat G(s) \rbrack^\top$
            \\
    		$\mathbf{\hat A} \leftarrow - (\mathbb{E}[\mathbf{B}^\top \mathbf{B}])^{-1}\mathbb{E}[\mathbf{B}^\top \mathbf{C}]$
    		\\
    		$\mathbf{A} \leftarrow \eta_\lambda \mathbf{A} + (1 - \eta_\lambda) \mathbf{\hat A} $
    		\Comment{Update $\lambda$'s}
		}  
		\label{alg1}  
		\caption{Stochastic Action Set Policy Gradient (SAS-PG)}
	\end{algorithm2e}
	\DecMargin{1em}    

\begin{prop}[Unbiased estimator]
\label{prop:unbiased}
Let $\hat J(s, \alpha, a, \theta) \coloneqq \psi^\theta(s,\alpha, a)\left(q^\theta(s, a) + \lambda_1 \hat v(s) + \lambda_2 \bar q(s, \alpha) \right)$ and  $d^\pi(s) \coloneqq (1-\gamma)\sum_t^\infty \gamma^t \Pr(S_t=s)$, then for any values of $\lambda_1 \in \mathbb{R}$ and $\lambda_2 \in \mathbb{R}$,
\begin{align}
    %
    \nabla J(\theta) &= \mathbb{E}\left[ \hat J(s, \alpha, a, \theta) \middle| d^\pi, \varphi, \pi \right].
\end{align}
\end{prop}
\begin{proof}
See Appendix D. 
\end{proof}

\begin{figure*}[t]
		\centering
		\includegraphics[width=0.27\textwidth]{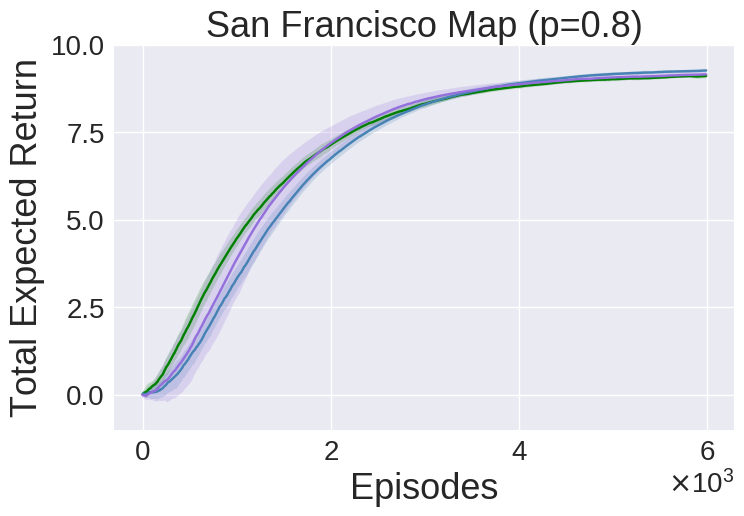} \hfill
		\includegraphics[width=0.27\textwidth]{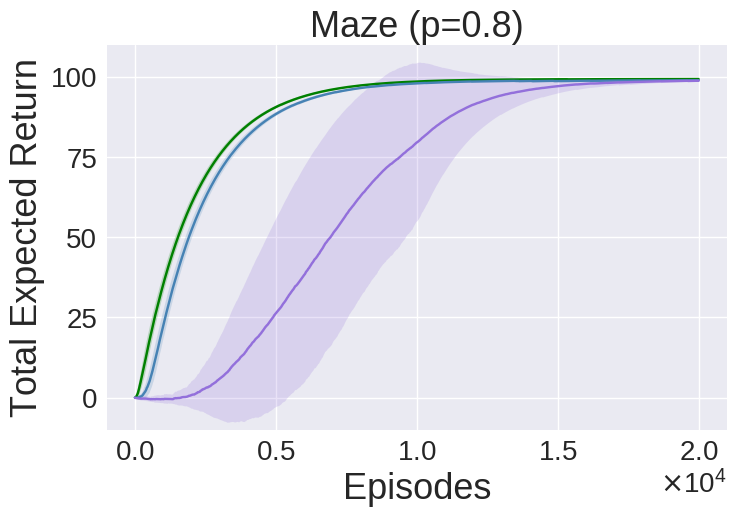} \hfill
		\includegraphics[width=0.27\textwidth]{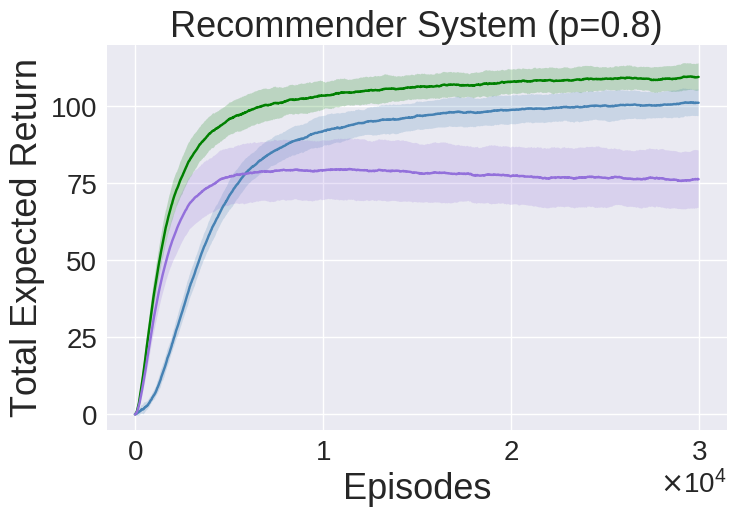} \hfill
		\includegraphics[width=0.12\textwidth]{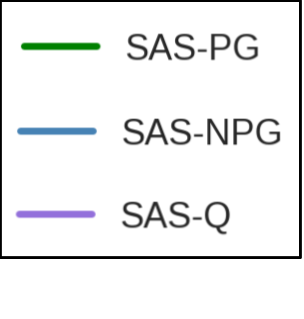}
		\includegraphics[width=0.27\textwidth]{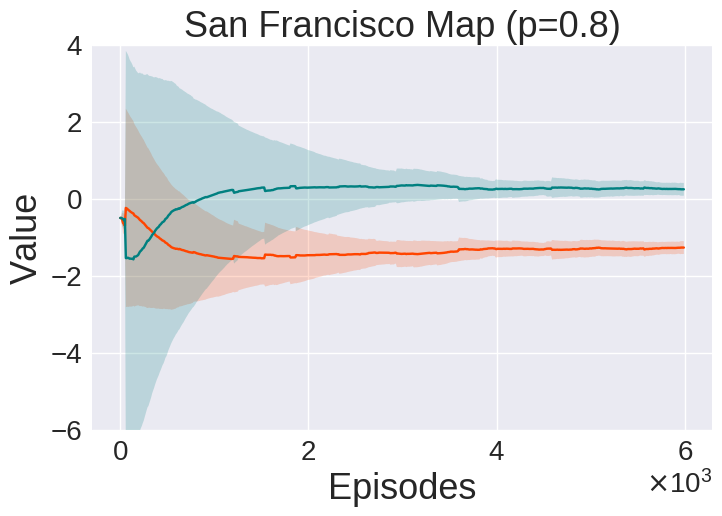} \hfill
		\includegraphics[width=0.27\textwidth]{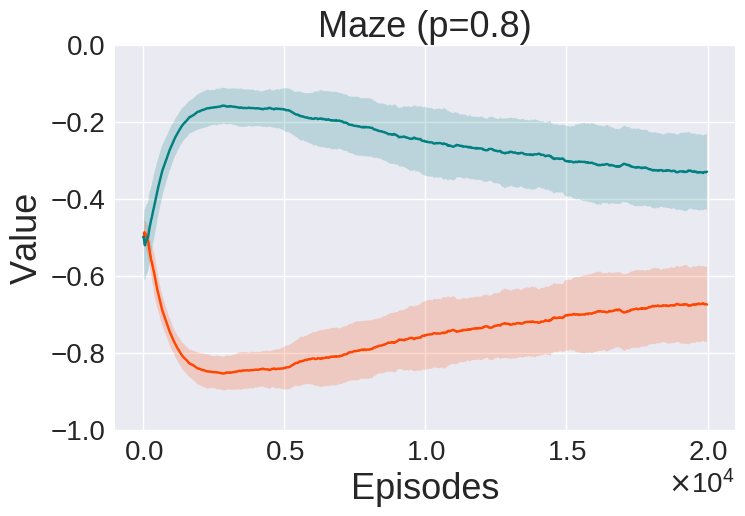} \hfill
		\includegraphics[width=0.27\textwidth]{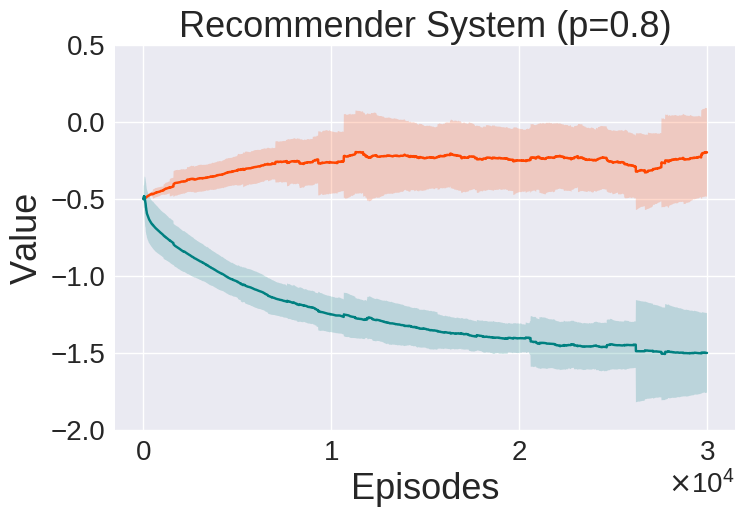} \hfill
		\includegraphics[width=0.12\textwidth]{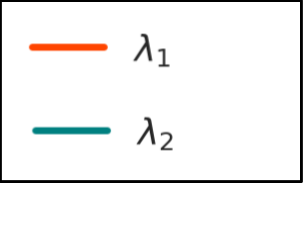}
		\caption{(Top) Best performing learning curves on the domains considered.
		The probability of any action being available in the action set is $0.8$. 
		(Bottom) Autonomously adapted values of $\lambda_1$ and $\lambda_2$ associated with $\hat v$ and $\bar q$, respectively, for the SAS-PG results.
		Shaded regions correspond to one standard deviation obtained using $30$ trials.}
		\label{Fig:plots}
\end{figure*}	
The question remains: what values should be used for $\lambda_1$ and $\lambda_2$ for combining $\hat v$ and $\bar q~$? 
Similar problems of combining different estimators has been studied in statistics literature \citep{graybill1959combining,meir1994bias} and more recently for combining control variates \citep{wang2013variance,geffner2018using}.  
Building upon their ideas, rather than leaving $\lambda_1$ and $\lambda_2$ as open hyperparameters, we propose a method for automatically adapting $\mathbf A=[\lambda_1,\lambda_2]$ for the specific SAS-MDP and current policy parameters, $\theta$. 
%
The following lemma presents an analytic expression for the value of $\mathbf A$ that minimizes a sample-based estimate of the variance of $\hat J$.
\begin{lemma}[Adaptive variance mitigation]
\label{lemma:variance}
If $\mathbf{A} = \lbrack \lambda_1, \lambda_2 \rbrack^\top,$ $\mathbf{B} = \lbrack  \psi^\theta(s,\alpha, a) \hat v(s),\psi^\theta(s,\alpha, a) \bar q(s, \alpha)\rbrack^\top ,$
    and $\mathbf{C} = \lbrack \psi^\theta(s,\alpha, a) q^\theta(s, a) \rbrack^\top$,
where $\mathbf{A} \in \mathbb{R}^{2\times 1}, \mathbf{B} \in \mathbb{R}^{d\times 2}$, and $\mathbf{C} \in \mathbb{R}^{d\times 1}$, then the $\mathbf{A}$ that minimizes the variance of $\hat J$ is given by
\begin{align}
    \mathbf{A} = - \left( \mathbb{E}\middle[\mathbf{B}^\top\mathbf{B}\middle] \right)^{-1} \mathbb{E} \left[ \mathbf{B}^\top\mathbf{C}  \right]. \label{eqn:minvar}
\end{align}
\end{lemma}
\begin{proof}
See Appendix D.
\end{proof}

Lemma \ref{lemma:variance} provides the values for $\lambda_1$ and $\lambda_2$ that result in the minimal variance of $\hat J$. 
Note that the computational cost associated with evaluating the inverse of $\mathbb{E}\left[\mathbf{B}^\top\mathbf{B}\right]$ is negligible because its dimension is always $\mathbb{R}^{2\times 2}$, independent of the number of policy parameters.
Also, Lemma \ref{lemma:variance} provides the optimal values of $\lambda_1$ and $\lambda_2$, which still must be approximated using sample-based estimates of $\mathbf B$ and $\mathbf C$. 
Furthermore, one might use double sampling for $\mathbf B$ to get unbiased estimates of the variance minimizing value of $\mathbf A$ \citep{baird1995residual}. 
However, as Property \ref{prop:unbiased} ensures that estimates of $\hat J$ for any value of $\lambda_1$ and $\lambda_2$ are always unbiased, we opt to use all the available samples for estimating $\mathbb E[\mathbf B^\top \mathbf B]$ and $\mathbb{E}[\mathbf B^\top \mathbf C]$. 
%
%
%

\section{Algorithm}
\label{apx:algorithms}
Pseudo-code for the SAS policy gradient algorithm is provided in Algorithm \ref{alg1}.
Let the estimators of $v^\theta$ and $q^\theta$ be
$\hat v^\varpi$ and $\hat q^\omega$, which are parameterized using $\varpi$ and $\omega$, respectively. 
Let $\pi^\theta$ corresponds to the policy parameterized using $\theta$.
Let $\eta_\varpi, \eta_\omega, \eta_\theta$ and $\eta_\lambda$ be the learning-rate hyper-parameters.
We begin by initializing the $\lambda$ values to $-0.5$ each, such that it takes an average of both the baselines and subtracts it off from the sampled return.
In Lines $3$ and $4$, we execute $\pi^\theta$ to observe the trajectory and compute the return.
%
%
Lines $6$ and $7$ correspond to the updates for parameters associated with $\hat v^\varpi$ and $\hat q^\omega$, using their corresponding TD errors \citep{sutton2018reinforcement}.
The policy parameters are then updated using  a  combination of both the baselines.
We drop the $\gamma^t$ dependency for data efficiency \citep{thomas2014bias}.
As per Lemma \ref{lemma:variance}, for automatically tuning the values of $\lambda_1$ and $\lambda_2$, we create the sample estimates of the matrices $\mathbf{B}$ and $\mathbf{C}$ using the transitions from batch $\mathbb{B}$, in Lines $9$ and $10$.
To update the values of $\lambda$'s, we compute $\mathbf{\hat A}$ using the sample estimates of $\mathbb{E}[\mathbf{B}^\top \mathbf{B}]$ and $\mathbb{E}[\mathbf{B}^\top \mathbf{C}]$.
While computing the inverse, a small diagonal noise is added to ensure that inverse exists.
As everything is parameterized using smooth function, we know that the subsequent estimates of $\mathbf{A}$ should not vary a lot.
Since we only have access to the sample estimate of $\mathbf{A}$, we leverage the Polyak-Rupert averaging  in Line $12$ for  stability.
Due to space constraints, the algorithm for SAS natural policy gradient is deferred to Appendix E. 

\begin{figure*}[t]
		\centering
		\includegraphics[width=0.27\textwidth]{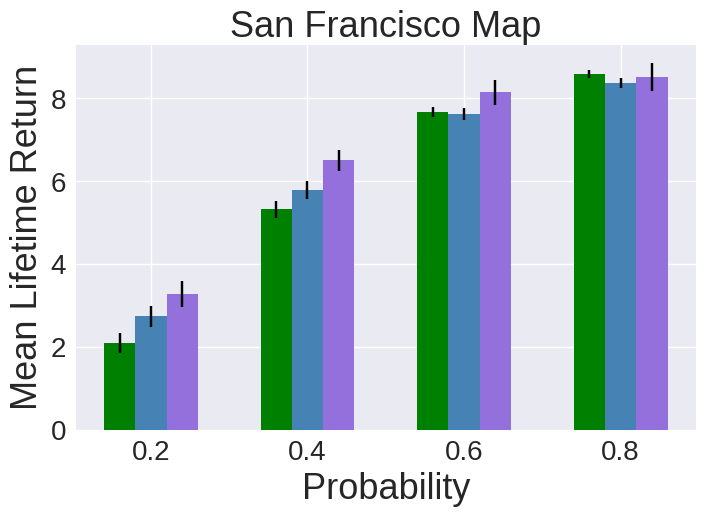} \hfill
		\includegraphics[width=0.27\textwidth]{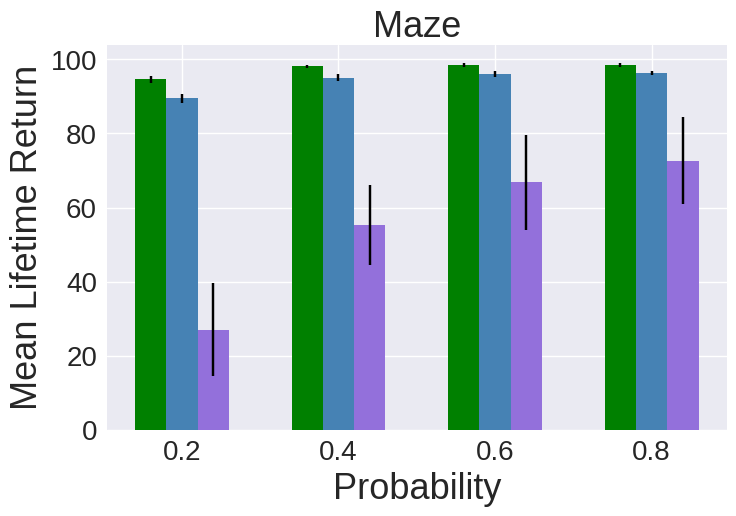} \hfill
		\includegraphics[width=0.27\textwidth]{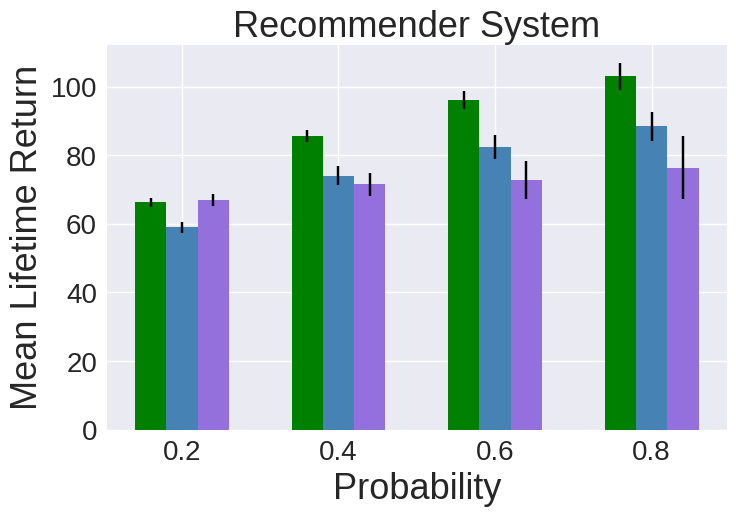} \hfill
		\includegraphics[width=0.12\textwidth]{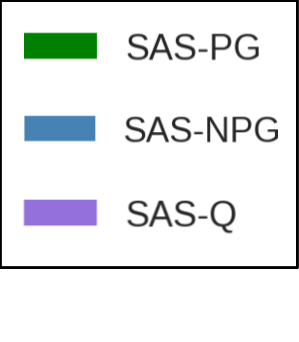}
		\caption{Best performances of different algorithms across different values of probabilities for action availability. The error bars correspond to one standard deviation obtained using $30$ trials. }
		\label{Fig:ablation}
\end{figure*}

\section{Empirical Analysis}
In this section we use empirical studies to answer the following three questions: 
(\textbf{a}) How do our proposed algorithms, SAS policy gradient (SAS-PG) and SAS natural policy gradient (SAS-NPG), compare to the prior method SAS-Q-learning?
(\textbf{b}) How does our adaptive variance reduction technique weight the two baselines over the training duration?
(\textbf{c}) What impact does the probability of action availability have on the performances of SAS-PG, SAS-NPG, and SAS-Q-learning? 
%
%
To evaluate these aspects, we first briefly introduce three domains inspired by real-world problems.
%
%
%
\paragraph{Routing in San Francisco. } 
This task models the problem of finding shortest paths in San Francisco, and was first presented with stochastic actions by \citet{Boutilier2018PlanningAL}. 
Stochastic actions model the concept that certain paths in the road network may not be available at certain times.
%
%
A positive reward is provided to the agent when it reaches the destination, while a small penalty is applied at every time step.
%
We modify the domain presented by \citet{Boutilier2018PlanningAL} so that the starting state of the agent is not one particular node, but rather is uniformly randomly chosen among all possible locations.
%
This makes the problem more challenging, since it requires the agent to learn the shortest path from every node. 
All the states (nodes) are discrete, and edges correspond to the action choices. 
Each edge is made available with some fixed probability.
The overall map is shown in Appendix.
%
%
%
\paragraph{Robot locomotion task in a maze. } 
In this domain, the agent has to navigate a maze using unreliable actuators.
The agent starts at the bottom left corner and a goal reward is given when it reaches the goal position, marked by a star (see Appendix for the figure).
The agent is penalized at each time step to encourage it to reach the goal as quickly as possible. 
The state space is continuous, and corresponds to real-valued Cartesian coordinates of the agent's position.
The agent has $16$ actuators pointing in different directions. 
Turning each actuator on moves the agent in the direction of the actuator.
However, each actuator is unreliable, and is therefore only available with some fixed probability.
%
\paragraph{Product recommender system.}
In online marketing and sales, product recommendation is a popular problem.
Due to various factors such as stock outage, promotions, delivery issues etc., not all products can be recommended always.
To model this, we consider a synthetic setup of providing recommendation to a user from a batch of $100$ products, each available with some fixed probability and associated with a stochastic reward corresponding to profit.
Each user has a real-valued context, which forms the state space, and the recommender system interacts with a randomly chosen user for $5$ steps.
The goal for the recommender system is to suggest products that maximize total profit.
Often the problem of recommendation is formulated as a contextual bandit or collaborative filtering problem, but as shown by \citet{theocharous2015ad} these approaches fail to capture the long term value of the prediction.
Hence we resort to the full RL setup.

\subsection{Results}
\label{sec:results}

Here we only discuss the representative results for the three major questions of interest.
Plots for detailed evaluations are available in Appendix F.
%

\textbf{(a)}
For the routing problem in San Francisco, as both the states and actions are discrete, the q-function for each state-action pair has a unique  parameter.
When no parameters are shared, SAS-Q-learning will not diverge.
Therefore, in this domain, we notice that SAS-Q-learning performs similarly to the proposed algorithms. 
However, in many large-scale problems, the use of function approximators is crucial for estimating the optimal policy.
For the robot locomotion task in the maze domain and the recommender system, the state space is not discrete and hence function approximators are required to obtain the state features.
As we saw in Section \ref{sec:limitations}, the sharing of state features can create problems for SAS-Q-learning.
The increased variance in the performance of SAS-Q-learning is visible in both the Maze and the Recommender system domains in Figure \ref{Fig:plots}.
While the SAS-Q eventually performs the same on the Maze domain, its performance improvement saturates quickly in the recommender system domain thus resulting in a sub-optimal policy.

\textbf{(b)}
To provide visual intuition for the behavior of adaptive variance mitigation, we report the values of $\lambda_1$ and $\lambda_2$ over the training duration in Figure \ref{Fig:plots}.
%
%
As several factors are combined through \eqref{eqn:minvar} to influence the $\lambda$ values, it is hard to pinpoint any individual factor that is responsible for the observed trend.
However, note that for both the routing problem in San Francisco and the robot navigation in maze, the goal reward is obtained on reaching the destination and intermediate actions do not impact the total return significantly.
Intuitively, this makes the action set conditioned baseline $\bar q$ similarly correlated to the observed return as the state only conditioned baseline, $\hat v$, but at the expense of estimating significantly more number of parameters.
Thus the importance for $\bar q$ is automatically adapted to be closer to zero.
On the other hand, in recommender system, each product has a significant amount of associated reward. 
Therefore, the total return possible during each episode has a strong dependency on the available action set and thus the magnitude of weight for $\bar q$ is much larger than that for $v$.
%

\textbf{(c)}
To understand the impact of the probability of an action being available, we report the best performances for all the algorithms for different probability values in  Figure \ref{Fig:ablation}.
We notice that in the San Francisco routing domain, SAS-Q-learning has a slight edge over the proposed methods. 
This can be attributed to the fact that off-policy samples can be re-used without causing any divergence problems as state features are not shared.
For the maze and the recommender system tasks, where function approximators are necessary, the proposed methods significantly out-perform SAS-Q.

\section{Conclusion}

Building upon the SAS-MDP framework of \citet{Boutilier2018PlanningAL}, we studied an under-addressed problem of dealing with MDPs with stochastic action sets.
We highlighted some of the limitations of the existing method and addressed them by generalizing policy gradient methods for SAS-MDPs.
Additionally, we introduced a novel baseline and an adaptive variance reduction technique unique to this setting.
Our approach has several benefits.
Not only does it generalize the theoretical properties of standard policy gradient methods, but it is also practically efficient and simple to implement.

\section{Acknowledgement}

The research was supported by and partially conducted at Adobe
Research. 
We are also immensely grateful to the three anonymous reviewers who shared their insights and feedback, specially to the second reviewer who helped improve the counter example. 

{\footnotesize
\bibliographystyle{aaai}
\bibliography{bibliography}
}
\appendix
\onecolumn

\setcounter{lemma}{0}
\setcounter{thm}{0}
\setcounter{cor}{0}
\setcounter{prop}{0}

\section*{\centering Reinforcement Learning When All Actions are Not \\
Always Available (Supplementary Material)}

\section{A: SAS Policy Gradient}
\label{apx:SASPG}
\begin{lemma}[SAS Policy Gradient]
For all $s \in \mathcal S$,
    \begin{equation}
    \label{apx:eq:SASPG}
       \frac{d}{d\theta} J(\theta) = \sum_{t=0}^{\infty} \sum_{s \in \mathcal S}\gamma^t \Pr(S_{t}=s|\theta)\sum_{\alpha \in 2^{\mathcal{B}}} \varphi(s, \alpha) \sum_{a \in \alpha} q^\theta(s,a) \frac{\partial \pi^\theta(s,\alpha, a)}{\partial \theta}.
    \end{equation}
\end{lemma}
\begin{proof} 

\begin{align}
    \frac{\partial v^\theta(s)}{\partial \theta} =& \frac{\partial}{\partial \theta} \mathbb{E}\left [ \sum_{k=0}^\infty \gamma^k R_{t} \middle | S_t=s, \theta \right ]\\
    =& \frac{\partial}{\partial \theta} \sum_{\alpha \in 2^\mathcal{B}}\!\!\!\varphi(s, \alpha)\!\! \sum_{a \in \alpha _s} \Pr(A_t=a|S_t=s, \mathcal A_t=\alpha, \theta) \mathbb{E}\left [ \sum_{k=0}^\infty \gamma^k R_{t+k} \middle | S_t=s, A_t=a,\theta \right ]\\
    =&\frac{\partial}{\partial \theta} \sum_{\alpha \in 2^\mathcal{B}}\!\!\!\varphi(s, \alpha)\!\!\sum_{a \in \alpha} \pi^\theta(s,\alpha, a) q^\theta(s,a)\\
    =& \sum_{\alpha \in 2^\mathcal{B}}\!\!\!\varphi(s, \alpha)\!\!  \sum_{a \in \alpha} \left (\frac{\partial \pi^\theta(s,\alpha,a)}{\partial \theta} q^\theta(s,a) +  \pi^\theta(s,\alpha, a) \frac{\partial q^\theta(s,a)}{\partial \theta} \right)\\
    =& \sum_{\alpha \in 2^\mathcal{B}}\!\!\!\varphi(s, \alpha)\!\! \sum_{a \in \alpha} \frac{\partial \pi^\theta(s,\alpha, a)}{\partial \theta} q^\theta(s,a) \\ &+\sum_{\alpha \in 2^\mathcal{B}}\!\!\!\varphi(s, \alpha)\!\!\sum_{a \in \alpha} \pi^\theta(s,\alpha, a) \frac{\partial}{\partial \theta}  \sum_{s' \in \mathcal S} P(s,a,s') \left ( R(s,a) + \gamma v^\theta(s') \right ) \label{apx:eq:unroll}\\
    =&\sum_{\alpha \in 2^\mathcal{B}}\!\!\!\varphi(s, \alpha)\!\!\sum_{a \in \alpha}\frac{\partial \pi^\theta(s,\alpha ,a)}{\partial \theta} q^\theta(s,a) +  \gamma \sum_{s' \in \mathcal S}\Pr(S_{t+1}=s'|S_t=s,\theta)  \frac{\partial v^\theta(s')}{\partial \theta},
    \end{align}
    where \eqref{apx:eq:unroll} comes from unrolling the Bellman equation. 
    We started with the partial derivative of the value of a state, expanded the definition of the value of a state, and obtained an expression in terms of the partial derivative of the value of another state. 
    Now, we again expand $\partial v^\theta(s') / \partial \theta$ using the definition of the state-value function and the Bellman equation. 
    \begin{align}
    \frac{\partial v^\theta(s)}{\partial \theta} =&\sum_{\alpha \in 2^\mathcal{B}}\!\!\!\varphi(s, \alpha)\!\!\sum_{a \in \alpha}\frac{\partial \pi^\theta(s,\alpha ,a)}{\partial \theta} q^\theta(s,a) \\
    &+  \gamma \sum_{s' \in \mathcal S}\Pr(S_{t+1}=s'|S_t=s,\theta)  \frac{\partial}{\partial \theta} \left ( \sum_{\alpha' \in 2^\mathcal{B}}\!\!\!\varphi(s', \alpha')\!\! \sum_{a' \in \alpha'} \pi^\theta(s',\alpha', a') q^\theta(s',a') \right )
    \\
    =&\sum_{\alpha \in 2^\mathcal{B}}\!\!\!\varphi(s, \alpha)\!\!\sum_{a \in \alpha}\frac{\partial \pi^\theta(s,\alpha ,a)}{\partial \theta} q^\theta(s,a) \\
    &+  \gamma \sum_{s' \in \mathcal S}\Pr(S_{t+1}=s'|S_t=s,\theta) \!\!\!\!\sum_{\alpha' \in 2^\mathcal{B}}\!\!\!\varphi(s', \alpha')\!\! \left ( \sum_{a' \in \alpha'} \frac{\partial \pi^\theta(s',\alpha', a')}{\partial \theta} q^\theta(s',a') + \pi^\theta(s',\alpha', a')\frac{\partial q^\theta(s',a')}{\partial \theta}\right )
    \end{align}
    \begin{align}
    =&\sum_{\alpha \in 2^\mathcal{B}}\!\!\!\varphi(s, \alpha)\!\!\sum_{a \in \alpha}\frac{\partial \pi^\theta(s,\alpha ,a)}{\partial \theta} q^\theta(s,a)   \\
    &+ \gamma \sum_{s' \in \mathcal S}\Pr(S_{t+1}=s'|S_t=s,\theta) \!\!\!\!\sum_{\alpha' \in 2^\mathcal{B}}\!\!\!\varphi(s', \alpha')\!\! \sum_{a' \in \alpha'} \frac{\partial \pi^\theta(s',\alpha', a')}{\partial \theta} q^\theta(s',a') \\
    &+ \gamma \sum_{s' \in \mathcal S}\Pr(S_{t+1}=s'|S_t=s,\theta) \!\!\!\!\sum_{\alpha' \in 2^\mathcal{B}}\!\!\!\varphi(s', \alpha')\!\!  \sum_{a' \in \alpha'} \!\!\!\!\pi^\theta(s',\alpha', a')\frac{\partial }{\partial \theta} \left ( \sum_{s'' \in \mathcal S} P(s',a',s'') (R(s',a') + \gamma v^\theta(s''))\right )\\
    =&\sum_{\alpha \in 2^\mathcal{B}}\!\!\!\varphi(s, \alpha)\!\!\sum_{a \in \alpha}\frac{\partial \pi^\theta(s,\alpha ,a)}{\partial \theta} q^\theta(s,a)   \\
    &+ \gamma \sum_{s' \in \mathcal S}\Pr(S_{t+1}=s'|S_t=s,\theta) \!\!\!\!\sum_{\alpha' \in 2^\mathcal{B}}\!\!\!\varphi(s', \alpha')\!\! \sum_{a' \in \alpha'} \frac{\partial \pi^\theta(s',\alpha', a')}{\partial \theta} q^\theta(s',a') \\
    &+ \gamma \sum_{s' \in \mathcal S}\Pr(S_{t+1}=s'|S_t=s,\theta) \!\!\!\!\sum_{\alpha' \in 2^\mathcal{B}}\!\!\!\varphi(s', \alpha')\!\!  \sum_{a' \in \alpha'} \!\!\!\!\pi^\theta(s',\alpha', a') \sum_{s'' \in \mathcal S} P(s',a',s'') \gamma \frac{\partial v^\theta(s'')}{\partial \theta} \\
    =&\underbrace{\sum_{\alpha \in 2^\mathcal{B}}\!\!\!\varphi(s, \alpha)\!\!\sum_{a \in \alpha}\frac{\partial \pi^\theta(s,\alpha ,a)}{\partial \theta} q^\theta(s,a)}_{\text{first term}}   \\
    &+ \underbrace{\gamma \sum_{s' \in \mathcal S}\Pr(S_{t+1}=s'|S_t=s,\theta) \!\!\!\!\sum_{\alpha' \in 2^\mathcal{B}}\!\!\!\varphi(s', \alpha')\!\! \sum_{a' \in \alpha'} \frac{\partial \pi^\theta(s',\alpha', a')}{\partial \theta} q^\theta(s',a')}_{\text{second term}} \\
    &+ \gamma^2 \sum_{s'' \in \mathcal S} \Pr(S_{t+2}=s'' | S_t=s,\theta) \frac{\partial v^\theta(s'')}{\partial \theta}. \label{apx:eqn:end1}
\end{align}
Expanding $\partial v^\theta(s') / \partial \theta$ allowed us to write it in terms of the partial derivative of yet another state, $s''$. 
We could continue this process, ``unravelling'' the recurrence further. 
Each time that we expand the partial derivative of the value of a state with respect to the parameters, we get another term. 
The first two terms that we have obtained are marked above. 
If we were to unravel the expression more times, by expanding $\partial v^\theta(s'') / \partial \theta$ and then differentiating, we would obtain the subsequent third, fourth, etc., terms.

Finally, to get the desired result, we expand the start-state objective and take the derivative with respect to it,
\begin{align}
   \frac{d}{d\theta} J(\theta) &= \sum_{s \in \mathcal{S}} d_0(s)  \frac{\partial}{\partial \theta} v^\theta(s). \label{apx:eqn:end2}
\end{align}

Combining results from \eqref{apx:eqn:end1} and \eqref{apx:eqn:end2}, we index each term by $t$, with the first term being $t=0$, the second $t=1$, etc., which results in the expression:
    \begin{align}
        \frac{d}{d\theta} J(\theta)  &=\sum_{t=0}^{\infty} \sum_{s \in \mathcal S}\gamma^t \Pr(S_{t}=s|\theta)\sum_{\alpha \in 2^{\mathcal{B}}} \varphi(s, \alpha) \sum_{a \in \alpha} q^\theta(s,a) \frac{\partial \pi^\theta(s,\alpha, a)}{\partial \theta}.
    \end{align}
Notice that to get the gradient with respect to $J(\theta)$, we have included a sum over all the states weighted by, $d_0(s)$, the start state probability. 
When $t=0$, the only state where $\Pr(S_{0}=s|S_0=s,\theta)$ is not zero will be when $s=s$ (at which point this probability is one). 
This allows us to succinctly represent all the terms.
With this we conclude the proof.
\end{proof}

\section{B: Convergence}
%
%
%
\label{apx:convergence}
\begin{lemma}
Under Assumptions \eqref{apx:ass:1}-\eqref{apx:ass:3}, SAS policy gradient algorithm causes $\nabla J(\theta_t) \to 0$ as $t \to \infty$, with probability one.
\end{lemma}
\begin{proof}
    Following the standard result on convergence of gradient ascent (descent) methods \citep{bertsekas2000gradient}, we know that under Assumptions \eqref{apx:ass:1}-\eqref{apx:ass:3}, either $J(\theta) \to \infty$ or $\nabla J(\theta) \to 0$ as $t \to \infty$.
    However, maximum rewards possible is $R_\text{max}$ and $\gamma < 1$, therefore $J(\theta)$ is bounded above by $R_\text{max}/(1 - \gamma)$.
    Hence $J(\theta)$ cannot go to $\infty$ and we get the desired result.
\end{proof}

\section{C: SAS Natural Policy Gradient}
\label{apx:sasnpg}
\begin{prop}[Fisher Information Matrix]
\label{apx:prop:fim}
For a policy, parameterized using weights $\theta$, let $\psi^\theta(s, \alpha, a) \coloneqq \partial \log \pi^\theta (s, \alpha, a)/\partial \theta$, then the Fisher information matrix is,
\begin{align}
    F_\theta = \sum_{t=0}^{\infty} \sum_{s \in \mathcal S}\gamma^t \Pr(S_{t}=s|\theta)\!\!\sum_{\alpha \in 2^{\mathcal{B}}} \!\!\!\varphi(s, \alpha) \sum_{a \in \alpha} \pi^\theta(s,\alpha, a)\psi(s, \alpha, a) \psi(s, \alpha, a)^\top.
\end{align}
\end{prop}

\begin{proof}
To prove this result, we first note the following relation by \citet{amari2007methods} which connects the Hessian and the FIM of a random variable $X$ parameterized using $\theta$,
\begin{align}
    \mathbb{E}\left[\frac{\partial^2 \log \Pr(X)}{\partial \theta^2} \right]  = - \mathbb{E} \left [\frac{\partial \log \Pr(X)}{\partial \theta} \frac{\partial \log \Pr(X)}{\partial \theta}^\top \right ]. \label{apx:hess-fisher}
\end{align}
Now, let $\mathscr T_\theta$ denote the random variable corresponding to the \textit{trajectories} observed using policy $\pi^\theta$.
Let $\tau = (s_0, \alpha_0, a_0, s_1, \alpha_1, a_1, ...)$ denote an outcome of $\mathscr T_\theta$, then the probability of observing this trajectory, $\tau$, is given by,
\begin{align}
    \Pr(\mathscr T_\theta = \tau) &= \Pr(s_0)\prod_{t=0}^\infty \Pr(\alpha_t|s_t) \Pr(a_t | s_t, \alpha_t) \Pr(s_{t+1}|s_t, a_t)
    \\
    &= d_0(s_0)\prod_{t=0}^\infty \varphi(s_t, \alpha_t) \pi^\theta(s_t, \alpha_t, a_t)  P(s_t, a_t,s_{t+1}).
\end{align}
Therefore,
\begin{align}
    \frac{\partial^2}{\partial \theta^2} \log \Pr(\mathscr T_\theta = \tau) &= \frac{\partial^2}{\partial \theta^2}\log \left(  d_0(s_0)\prod_{t=0}^\infty \varphi(s_t, \alpha_t) \pi^\theta(s_t, \alpha_t, a_t) P(s_t, a_t,s_{t+1}) \right)
    \\
    &= \frac{\partial^2}{\partial \theta^2} \left( \log d_0(s_0) + \sum_{t=0}^\infty \log\varphi(s_t, \alpha_t)  + \sum_{t=0}^\infty \log\pi^\theta(s_t, \alpha_t, a_t) + \sum_{t=0}^\infty \log P(s_t, a_t,s_{t+1}) \right)
    \\
    &= \sum_{t=0}^\infty \frac{\partial^2}{\partial \theta^2} \log\pi^\theta(s_t, \alpha_t, a_t). \label{apx:eqn:trajectory-policy}
\end{align}
We know that Fisher Information Matrix for a random variable, which in our case is $\mathscr T_\theta$, is given by,
\begin{align}
    F_\theta &= \mathbb{E} \left[\frac{\partial \log \Pr(\mathscr T_\theta)}{\partial \theta} \frac{\partial \log \Pr(\mathscr T_\theta)}{\partial \theta} ^\top \right]
    \\
    &= - \mathbb{E} \left[\frac{\partial^2 \log \Pr(\mathscr T_\theta)}{\partial \theta^2}  \right] & \text{(Using Equation \eqref{apx:hess-fisher})}
    \\
    &= - \mathbb{E} \left[\sum_{t=0}^\infty \frac{\partial^2}{\partial \theta^2} \log\pi^\theta(s_t, \alpha_t, a_t) \right]   & \text{(Using Equation \eqref{apx:eqn:trajectory-policy})}
\\
    &= - \sum_{\tau \in \mathscr T_\theta} \Pr(\mathscr T_\theta = \tau)\sum_{t=0}^\infty \frac{\partial^2}{\partial \theta^2} \log\pi^\theta(s_t, \alpha_t, a_t)  , \label{apx:eqn:traj1}
\end{align}
where the summation over $\mathscr T_\theta$ corresponds to all possible values of $s, \alpha$ and $a$ for every step $t$ in the trajectory.
Expanding the inner summation in \eqref{apx:eqn:traj1},
\begin{align}
    F_\theta &= - \sum_{\tau \in \mathscr T_\theta} \Pr(\mathscr T_\theta = \tau) \frac{\partial^2}{\partial \theta^2} \log\pi^\theta(s_0, \alpha_0, a_0) - \sum_{\mathscr T_\theta} \Pr(\mathscr T_\theta) \frac{\partial^2}{\partial \theta^2} \log\pi^\theta(s_1, \alpha_1, a_1) - ...  \label{apx:eqn:traj2}   
\end{align}
Note that the summation in \eqref{apx:eqn:traj2} over all possible trajectories, i.e. all possible values of $s, \alpha$ and $a$ for every step $t$, marginalizes out the terms not associated with respective $\log \pi^\theta$ terms, i.e.,
\begin{align}
    F_\theta =& - \sum_{s_0 \in \mathcal S} \Pr(S_0=s_0|\theta)\!\!\sum_{\alpha_0 \in 2^{\mathcal{B}}} \!\!\!\varphi(s_0, \alpha_0) \sum_{a_0 \in \alpha_0} \pi^\theta(s_0,\alpha_0, a_0) \frac{\partial^2}{\partial \theta^2} \log\pi^\theta(s_0, \alpha_0, a_0) 
    \\
    &- \sum_{s_1 \in \mathcal S} \Pr(S_1=s_1|\theta)\!\!\sum_{\alpha_1 \in 2^{\mathcal{B}}} \!\!\!\varphi(s_1, \alpha_1) \sum_{a_1 \in \alpha_1} \pi^\theta(s_1,\alpha_1, a_1) \frac{\partial^2}{\partial \theta^2} \log\pi^\theta(s_1, \alpha_1, a_1)
    \\
    &- ...  \label{apx:eqn:traj3}  
\end{align}
Combining all the terms in \eqref{apx:eqn:traj3} and discounting them appropriately with $\gamma$, we get,
\begin{align}
    F_\theta = - \sum_{t=0}^{\infty} \sum_{s \in \mathcal S}\gamma^t \Pr(S_{t}=s|\theta)\!\!\sum_{\alpha \in 2^{\mathcal{B}}} \!\!\!\varphi(s, \alpha) \sum_{a \in \alpha} \pi^\theta(s,\alpha, a) \frac{\partial^2}{\partial \theta^2} \log\pi^\theta(s, \alpha, a). \label{apx:eqn:qw}
\end{align}
Finally, note that using \eqref{apx:hess-fisher},
\begin{align}
    \sum_{a \in \alpha} \pi^\theta(s,\alpha, a) \frac{\partial^2}{\partial \theta^2} \log\pi^\theta(s, \alpha, a) = - \sum_{a \in \alpha} \pi^\theta(s,\alpha, a) \psi(s, \alpha, a)\psi(s, \alpha, a)^\top. \label{apx:eqn:qwwq}
\end{align}
Combining \eqref{apx:eqn:qw} and \eqref{apx:eqn:qwwq} we get,
\begin{align}
     F_\theta = \sum_{t=0}^{\infty} \sum_{s \in \mathcal S}\gamma^t \Pr(S_{t}=s|\theta)\!\!\sum_{\alpha \in 2^{\mathcal{B}}} \!\!\!\varphi(s, \alpha) \sum_{a \in \alpha} \pi^\theta(s,\alpha, a) \psi(s, \alpha, a)\psi(s, \alpha, a)^\top.
\end{align}
With this we conclude the proof.
\end{proof}

\begin{lemma}[SAS Natural Policy Gradient]
Let $w$ be a parameter such that,
\begin{align}
   \frac{\partial }{\partial w} \mathbb{E} \left [\frac{1}{2} \sum_t^\infty \gamma^t \left(\psi(S_t, \mathcal A_t, A_t)^\top w - q^\theta(S_t, A_t)\right)^2 \right] = 0, \label{apx:eqn:fisher_grad}
\end{align}
then for all $s \in \mathcal S$ in $\mathcal M'$, 
\begin{align}
    \widetilde \nabla J(\theta) = w.
\end{align}
\end{lemma}

\begin{proof}
We begin by expanding \eqref{apx:eqn:fisher_grad},
\begin{align}
    \mathbb{E} \left [\sum_t^\infty \gamma^t \left(\psi(S_t, \mathcal A_t, A_t)^\top w - q^\theta(S_t, A_t)\right)\psi(S_t, \mathcal A_t, A_t) \right] &= 0
    \\
    \mathbb{E} \left [\sum_t^\infty \gamma^t \psi(S_t, \mathcal A_t, A_t) \psi(S_t, \mathcal A_t, A_t)^\top w  \right] &= \mathbb{E} \left [\sum_t^\infty \gamma^t \psi(S_t, \mathcal A_t, A_t) q^\theta(S_t, A_t)  \right]. \label{apx:eqn:fish1}
\end{align}
\begin{align}
   \widetilde \nabla J(\theta) &\coloneqq F^{-1}_\theta 
    \frac{\partial}{\partial \theta}J(\theta)
    \\
    &= F^{-1}_\theta \sum_{t=0}^{\infty} \sum_{s \in \mathcal S}\gamma^t \Pr(S_{t}=s|\theta)\sum_{\alpha \in 2^{\mathcal{B}}} \varphi(s, \alpha) \sum_{a \in \alpha} q^\theta(s,a) \frac{\partial \pi^\theta(s,\alpha, a)}{\partial \theta}
    \\
    &= F^{-1}_\theta \sum_{t=0}^{\infty} \sum_{s \in \mathcal S}\gamma^t \Pr(S_{t}=s|\theta)\sum_{\alpha \in 2^{\mathcal{B}}} \varphi(s, \alpha) \sum_{a \in \alpha} \pi^\theta(s,\alpha, a)  \psi^\theta (s,\alpha, a)q^\theta(s,a).  \label{apx:eqn:fish2}
\end{align}
Now combining \eqref{apx:eqn:fish1} and \eqref{apx:eqn:fish2},
\begin{align}
    \widetilde \nabla J(\theta) &= F^{-1}_\theta \sum_{t=0}^{\infty} \sum_{s \in \mathcal S}\gamma^t \Pr(S_{t}=s|\theta)\sum_{\alpha \in 2^{\mathcal{B}}} \varphi(s, \alpha) \sum_{a \in \alpha} \pi^\theta(s,\alpha, a)  \psi^\theta (s,\alpha, a)\psi^\theta (s,\alpha, a)^\top w
    \\
    &=  F^{-1}_\theta  F_\theta w
    \\
    &= w,
\end{align}
where the second last step follows from Property \ref{apx:prop:fim}. 
With this we conclude the proof.
\end{proof}

\section{D: Adaptive Variance Mitigation}
\label{apx:variance}
\begin{prop}[Unbiased estimator]
Let $\hat J(s, \alpha, a, \theta) \coloneqq \psi^\theta(s,\alpha, a)\left(q^\theta(s, a) + \lambda_1 \hat v(s) + \lambda_2 \bar q(s, \alpha) \right)$ and  $d^\pi(s) \coloneqq (1-\gamma)\sum_t^\infty \gamma^t \Pr(S_t=s)$, then for any values of $\lambda_1 \in \mathbb{R}$ and $\lambda_2 \in \mathbb{R}$,
\begin{align}
%
%
    \nabla J(\theta) &= \mathbb{E}\left[ \hat J(s, \alpha, a, \theta) \middle| d^\pi, \varphi, \pi \right].
\end{align}
\end{prop}

\begin{proof}
We begin by expanding $\nabla J(\theta)$,
\begin{align}
 \mathbb{E}\left[ \hat J(s, \alpha, a, \theta) \middle| d^\pi, \varphi, \pi \right] &=  \mathbb{E}\left[\psi^\theta(s,\alpha, a)\Big(q^\theta(s, a)\Big)\right] + \mathbb{E}\left[\psi^\theta(s,\alpha, a)\Big(\lambda_1 \hat v(s) +\lambda_2  \bar q(s, \alpha)\Big)\right]. \label{apx:eqn:unbiased-estmator}
    %
 \end{align}
 Now consider the term associated with the baselines $\hat v(s)$ and $\bar q$,
 \begin{align}
    &\mathbb{E}\left[\psi^\theta(s,\alpha, a)\Big(\lambda_1 \hat v(s) +\lambda_2  \bar q(s, \alpha)\Big)\right] \\
    &= \sum_{\alpha \in 2^\mathcal{B}, s \in \mathcal S} \Pr(s, \alpha) \sum_{a \in \alpha} \pi^\theta(s, \alpha, a) \frac{\partial \ln\pi^\theta(s, \alpha, a)}{\partial \theta}\Big(\lambda_1 \hat v(s) +\lambda_2  \bar q(s, \alpha)\Big) \\
    &= \sum_{\alpha \in 2^\mathcal{B}, s \in \mathcal S} \Pr(s, \alpha) \Big(\lambda_1 \hat v(s) +\lambda_2  \bar q(s, \alpha)\Big)\sum_{a \in \alpha} \pi^\theta(s, \alpha, a) \frac{\partial \ln\pi^\theta(s, \alpha, a)}{\partial \theta}. \label{apx:eqn:unb1}
\end{align}
Focusing only on the right part of \eqref{apx:eqn:unb1},
\begin{align}
    \sum_{a \in \alpha} \pi^\theta(s, \alpha, a) \frac{\partial \ln\pi^\theta(s, \alpha, a)}{\partial \theta} &= \sum_{a \in \alpha} \pi^\theta(s, \alpha, a) \frac{1}{\pi^\theta(s, \alpha, a)} \frac{\partial \pi^\theta(s, \alpha, a)}{\partial \theta}\\
    &= \sum_{a \in \alpha} \frac{\partial \pi^\theta(s, \alpha, a)}{\partial \theta}\\
    &= \frac{\partial}{\partial \theta}\sum_{a \in \alpha}\pi^\theta(s, \alpha, a)\\
    &= \frac{\partial}{\partial \theta}\mathbf{1}  
    \\
    &= 0. \label{apx:eqn:unb2}
\end{align}
Combining \eqref{apx:eqn:unb1} and \eqref{apx:eqn:unb2}, we observe that the bias of this new baseline combination is zero and we get the desired result.
\end{proof}

\begin{lemma}[Adaptive variance mitigation]
Let
\begin{align}
    \mathbf{A} &= \lbrack \lambda_1, \lambda_2 \rbrack^\top,
    \\
    \mathbf{B} &= \lbrack  \psi^\theta(s,\alpha, a) \hat v(s) , \psi^\theta(s,\alpha, a) \bar q(s, \alpha)\rbrack^\top ,
    \\
    \mathbf{C} &= \lbrack \psi^\theta(s,\alpha, a) q^\theta(s, a) \rbrack^\top,
\end{align}
such that, $\mathbf{A} \in \mathbb{R}^{2\times 1}, \mathbf{B} \in \mathbb{R}^{d\times 2}$ and $\mathbf{C} \in \mathbb{R}^{d\times 1}$, then the $\mathbf{A}$ that minimizes variance of $\hat J$ is given by,
\begin{align}
    \mathbf{A} = - \left( \mathbb{E}\middle[\mathbf{B}^\top\mathbf{B}\middle] \right)^{-1} \mathbb{E} \left[ \mathbf{B}^\top\mathbf{C}  \right].
\end{align}
\end{lemma}

\begin{proof}
Let the sample estimate for the gradient be given by,
\begin{align}
    \hat J(\theta) \coloneqq \hat J(s, \alpha, a, \theta) &= \psi^\theta(s,\alpha, a)\left( q^\theta(s, a) +\lambda_1 \hat v(s) +\lambda_2  \bar q(s, \alpha)\right). \label{apx:eqn:unbiased-estimate}
\end{align}
We aim to find the values of $\lambda$ that minimizes the variance of this estimator, i.e.,
\begin{align}
    \lambda &= \underset{\lambda}{\text{argmin}} \left[ \text{var}(\hat J(\theta))\right].
\end{align}

The variance of the estimator can be computed as following,
\begin{align}
    \text{var}(\hat J(\theta)) &= \mathbb{E}\left[ \middle(\hat J(\theta) -  \mathbb{E}\middle[\hat J(\theta) \middle] \middle)^\top \middle(\hat J(\theta) -  \mathbb{E}\middle[\hat J(\theta) \middle] \middle) \right] \\
    &= \mathbb{E}\left[ \hat J(\theta)^\top \hat J(\theta) \right]  -  2\mathbb{E}\left[\hat J(\theta)^\top \mathbb{E}\middle[\hat J(\theta) \middle] \middle]  +  \mathbb{E}\middle[\hat J(\theta)\middle]^\top \mathbb{E}\middle[\hat J(\theta)\right] \\
    &= \mathbb{E}\left[ \hat J(\theta)^\top \hat J(\theta) \right] - \mathbb{E}\left[\hat J(\theta)\middle]^\top \mathbb{E}\middle[\hat J(\theta)\right]. \label{apx:eqn:var1}
\end{align}
From Property \ref{prop:unbiased} we know that,
\begin{align}
    \mathbb{E}\left[\hat J(\theta)\right] &= \mathbb{E}\left[\psi^\theta(s,\alpha, a)\left( q^\theta(s, a) +\lambda_1 \hat v(s) +\lambda_2  \bar q(s, \alpha)\right)\right]  \\
    &= \mathbb{E}\left[\psi^\theta(s,\alpha, a)q^\theta(s, a) \right] + 0
    \\
    &=  \mathbb{E}\big[\mathbf{C}\big].
\end{align}
Expanding \eqref{apx:eqn:var1} in the matrix notations, 
\begin{align}
  \text{var}(\hat J(\theta)) &= \mathbb{E}\left [\left(\mathbf{C} + \mathbf{BA}^\top \right)^\top 
    \left(\mathbf{C} + \mathbf{BA}^\top \right)\right] - \mathbb{E}\big[\mathbf{C}\big]^\top \mathbb{E}\big[\mathbf{C}\big]  \\
    &= \mathbb{E} \mathbf{\left[C^\top C \right] + \mathbb{E}\left[C^\top BA^\top \right] + \mathbb{E}\left[AB^\top C \right] + \mathbb{E}\left[AB^\top BA^\top \right] - \mathbb{E}\big[C\big]^\top \mathbb{E} \big[C \big]}.
    \label{apx:eqn:var2}
\end{align}
Since the first and last term from \eqref{apx:eqn:var2} are independent of $\mathbf{A}$, it does not effect the optimization step. 
The remaining terms that matter are,
\begin{align}
    \mathbb{E} \mathbf{\left[C^\top BA^\top \right] + \mathbb{E}\left[AB^\top C \right] + \mathbb{E}\left[AB^\top BA^\top \right]}.
\end{align}
Differentiating these terms with respect to $\mathbf{A}$, and by equating it to $0$, we get,
\begin{align}
    2\mathbb{E}\mathbf{\left[B^\top C \right]} + 2\mathbb{E}\mathbf{\left[AB^\top B \right]} &= 0 
    \\
 2\mathbb{E}\mathbf{\left[AB^\top B \right]} &= - 2\mathbb{E}\mathbf{\left[B^\top C \right]} \\
 \mathbf{A}\mathbb{E}\mathbf{\left[B^\top B \right]} &= - \mathbb{E}\mathbf{\left[B^\top C \right]} \\
 \mathbf{ A}&= - \left(\mathbb{E}\mathbf{\left[B^\top B \right]  }\right)^{-1} \mathbb{E}\mathbf{\left[B^\top C \right]}
\end{align}
\end{proof}

\section{E: SAS Natural Policy Gradient}
Pseudo-code for SAS natural policy gradient is provided in Algorithm \ref{apx:Alg:2}.
Let the learning-rate for updating $\theta$ and $w$ be given by $\eta_\theta$ and $\eta_w$, respectively.
Similar to Algorithm 1, we first collect the transition batch $\mathbb{B}$ and compute the sampled returns from each state in Lines $2$ and $3$.
Following Lemma \ref{lemma:sasnpg}, we update the parameter $w$ in Line $5$ to minimize its associated TD error. 
The updated parameter $w$ is then used to update the policy parameters $\theta$.
As dividing by a scalar does not change the direction of the (natural) gradient,  we normalize the update using norm of $w$ in Line $6$ for better stability.

	\IncMargin{1em}
	\begin{algorithm2e}[h]
		\For {$episode = 0,1,2...$}{
			\nonl \textcolor[rgb]{0.5,0.5,0.5}{\# Collect transition batch  using $\pi^\theta$}
			\\
    		$\mathds{B} = \{(s_0, \alpha_0, a_0, r_0), ..., (s_T, \alpha_T, a_T, r_T)  \}$  
    		\\
    		$\hat G(s_t) = \sum_{k=0}^{T -t} \gamma^k r_{t+k}  $
    		\\
    		\vspace{8pt}
    		\nonl \textcolor[rgb]{0.5,0.5,0.5}{\# Perform batch update on parameters}
    		\\
    		$\psi^\theta (s, \alpha, a) = \frac{ \partial \log \pi^\theta(s, \alpha, a)}{\partial \theta}$
    		\\
    		$w \leftarrow w + \eta_w (\hat G(s) - \psi^\theta (s, \alpha, a)^\top w)\psi^\theta (s, \alpha, a)$ \Comment{ Update $w$}
    		\\
    		$\theta \leftarrow \theta + \eta_\theta \frac{w}{\lVert w \rVert }$ \Comment{ Update $\pi^\theta$}
		}  
		\caption{Stochastic Action Set Natural Policy Gradient (SAS-NPG)}
		\label{apx:Alg:2}  
	\end{algorithm2e}
	\DecMargin{1em}   

\section{F: Empirical Analysis Details}
\label{apx:empirical}
\begin{figure}[t]
		\centering
		\includegraphics[width=0.25\textwidth]{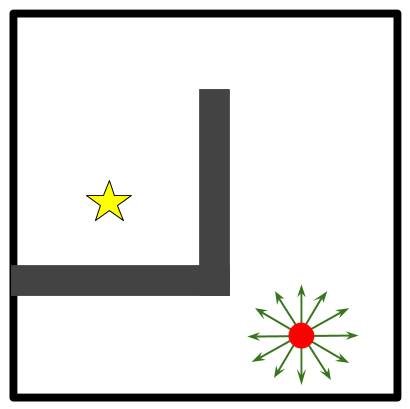} \hspace{60pt}
		\includegraphics[width=0.25\textwidth]{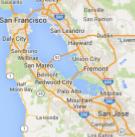}
		\caption{(Left) An illustrations of the top view of the maze domain, where the red dot corresponds to the agent and the green arrows around it represent the actions. The star represents the goal position. 
		(Right) Map view of San Francisco bay area.
		We consider the road network similar to the one used by \citet{Boutilier2018PlanningAL}. 
        }
		\label{Fig:domains}
\end{figure}

\subsection{Implementation details}
\label{apx:imp}
\paragraph{Policy parmaterization.} To make the policy handle stochastic action sets, we make use of a mask which indicates the available actions.
Formally, let $\phi(s) \in \mathbb{R}^d$ be the feature vector of the state and let $\theta \in \mathbb{R}^{d \times |\mathcal B|}$ denote the parameters that project the features on the space of all actions.
Let $y \coloneqq  \phi(s)^\top \theta$ denote the scores for each action and let $\mathds{1}_{\{a \in \alpha\}}$ be the indicator variable denoting whether the action $a$ is in the available action set $\alpha$ or not.
The probability of choosing an action is then computed using the masked softmax, i.e.,
\begin{align}
    \pi^\theta(s, \alpha, a) &\coloneqq \frac{\exp(y_a) \cdot \mathds{1}_{\{a \in \alpha\}}}{ \underset{a' \in \alpha}{\sum} \exp (y_{a'}) \cdot \mathds{1}_{\{a' \in \alpha\}}},
\end{align}
where $y_a$ corresponds to the score of action $a$ in $y$.

\paragraph{Hyperparamter settings.}
For the maze domain, state features were represented using $3^\text{rd}$ order coupled Fourier basis \citep{konidaris2011value}.
For the San Francisco map domain, one-hot encoding was used to represent each of the nodes (states) in the road-network.
For the recommender system domain, the user-context provided by the environment was directly used as state-features. 
Using these features, single layer-neural networks were used to represent the policy, baselines and the q-function for all the algorithms, for all the domains.
The discounting parameter $\gamma$ was set to $0.99$ for all the domains. 

For SAS policy gradient, the learning rates for both the baselines were searched over $[1e-2, 1e-4]$.
The learning rate for policy was searched over $[5e-3, 5e-5]$.
The hyper-parameter $\eta_\lambda$ was kept fixed to $0.999$ throughout. 
For SAS natural policy gradient, the learning rate, $\eta_w$, was searched over $[1e-2, 1e-4]$.

For SAS-Q-learning baseline, the exploration parameter for $\epsilon$-greedy was searched over $[0.05, 0.15]$ and the 
Learning rate for the q-function was searched over $[1e-2, 1e-4]$.
To encompass both online and batch learning for SAS-Q-learning, additional hyperparameter search was done over the batch-sizes $\{1, 8, 16\}$ and the number of batches $\{1, 8, 16\}$ per update to the q-function.
Note that when both the batch size and the number of batches is $1$, it becomes the online version \cite{Boutilier2018PlanningAL}. 
In total, $1000$ settings for each algorithm, for each domain, were uniformly sampled from the mentioned hyper-parameter ranges/sets.
Results from the best performing setting is reported in all the plots.
Each hyper-parameter setting was ran using $30$ different seeds to get the standard deviation of the performance.

\subsection{Additional Experimental Results}
In Figures  \ref{apx:Fig:all_perf} and  \ref{apx:Fig:all_prob} we report the learning curves and the adapted $\lambda_1$ and $\lambda_2$ values for all the domains under different probability values of action availability.
\label{apx:results}
\begin{figure}[t]
		\centering
		\includegraphics[width=0.32\textwidth]{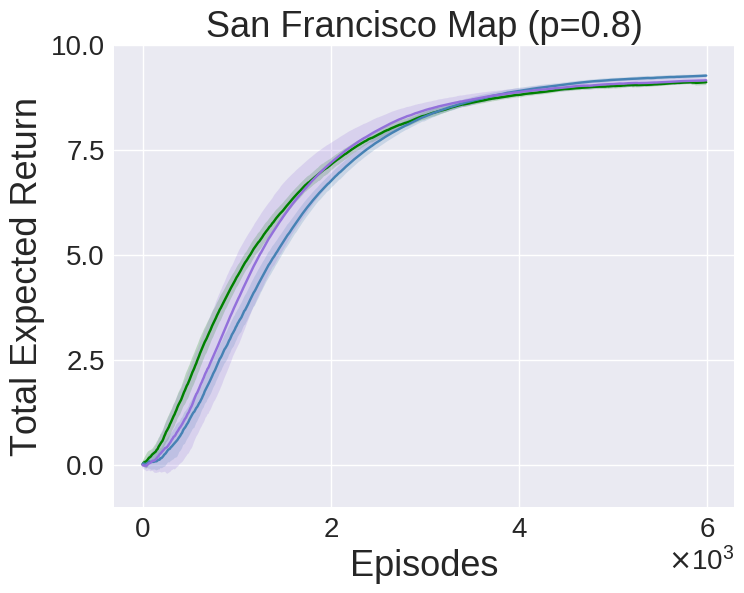}
		\includegraphics[width=0.32\textwidth]{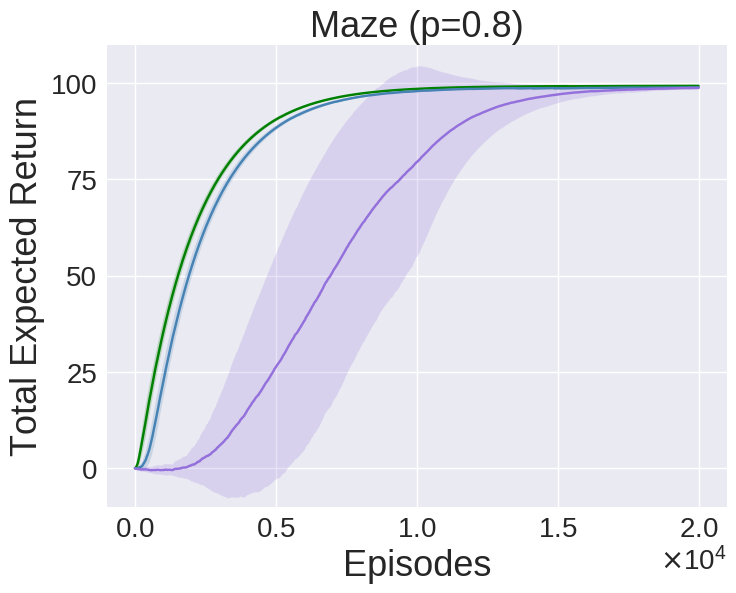}
		\includegraphics[width=0.32\textwidth]{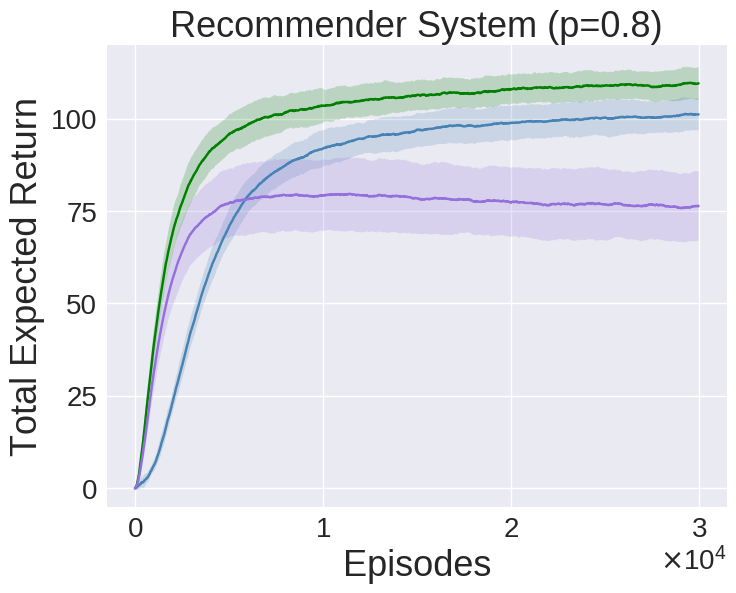}
		\\
		\includegraphics[width=0.32\textwidth]{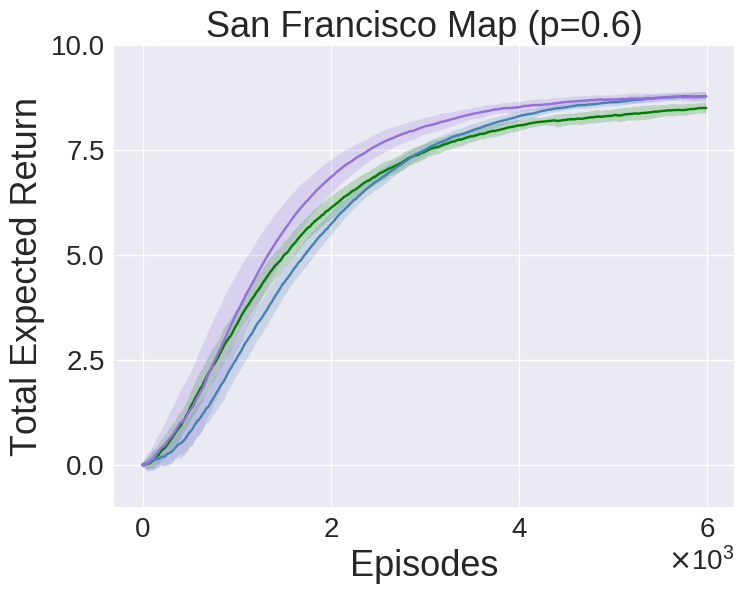}
		\includegraphics[width=0.32\textwidth]{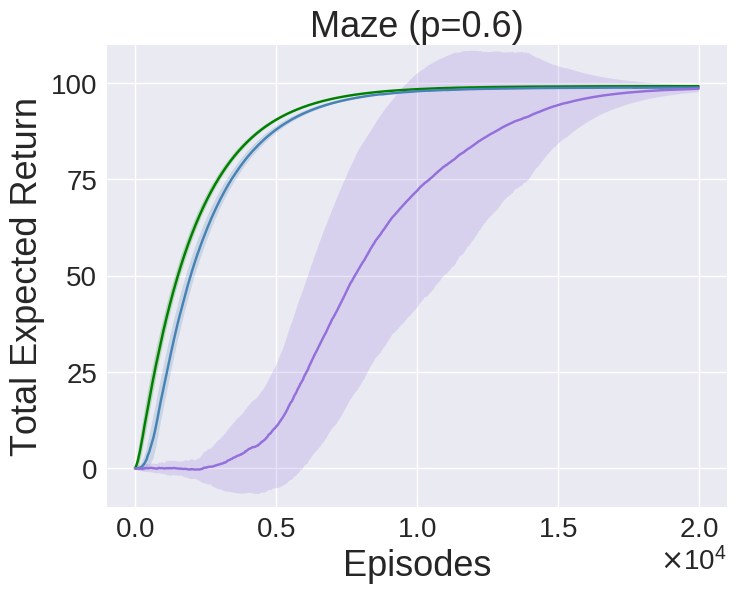}
		\includegraphics[width=0.32\textwidth]{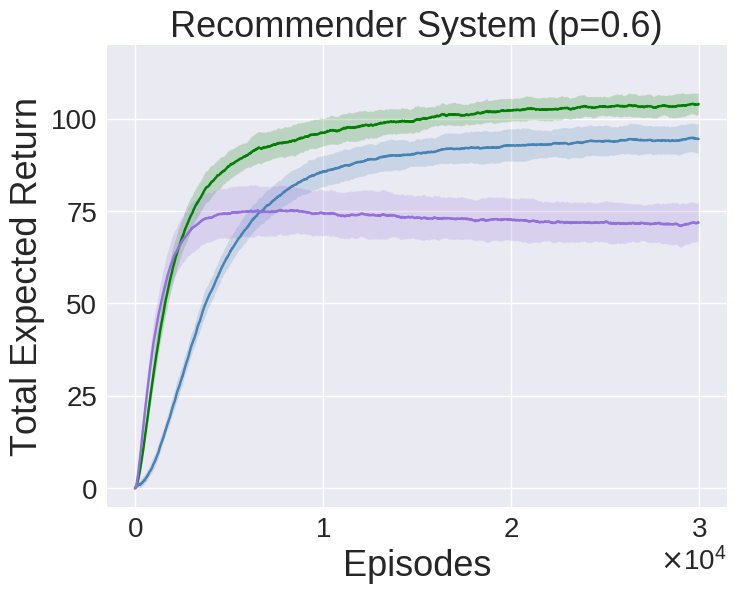}
		\\
		\includegraphics[width=0.32\textwidth]{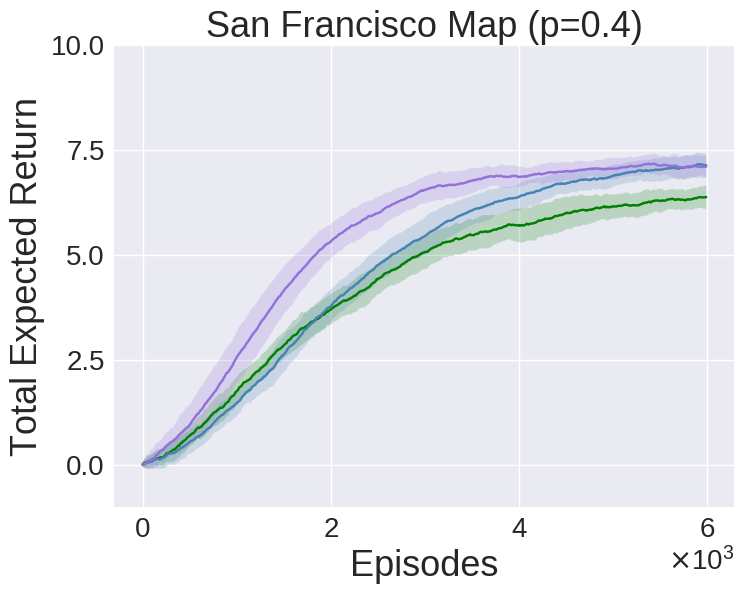}
		\includegraphics[width=0.32\textwidth]{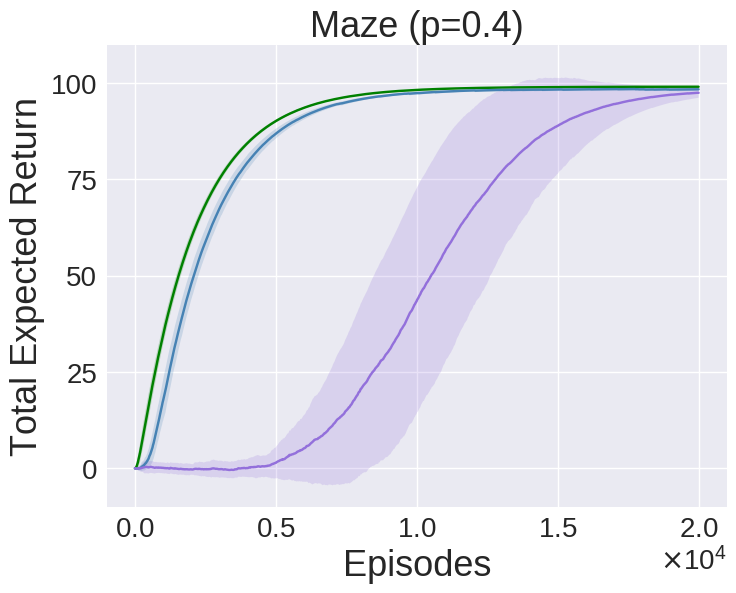}
		\includegraphics[width=0.32\textwidth]{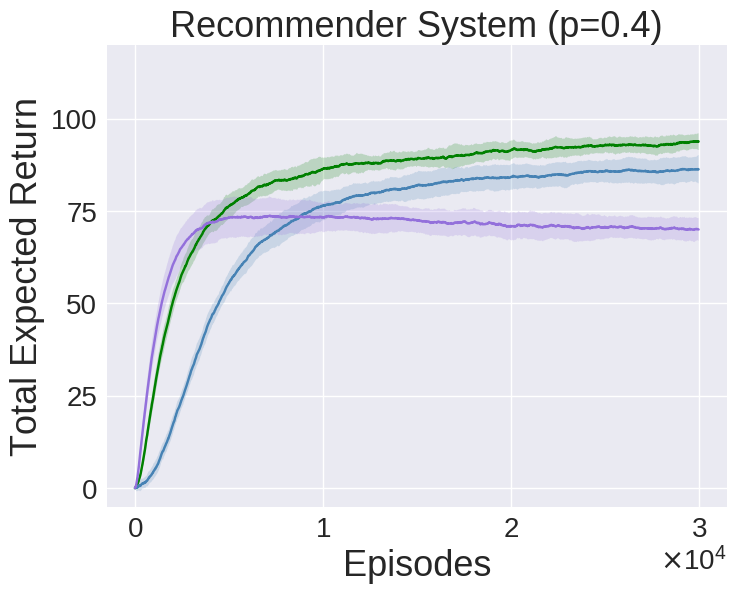}
		\\
		\includegraphics[width=0.32\textwidth]{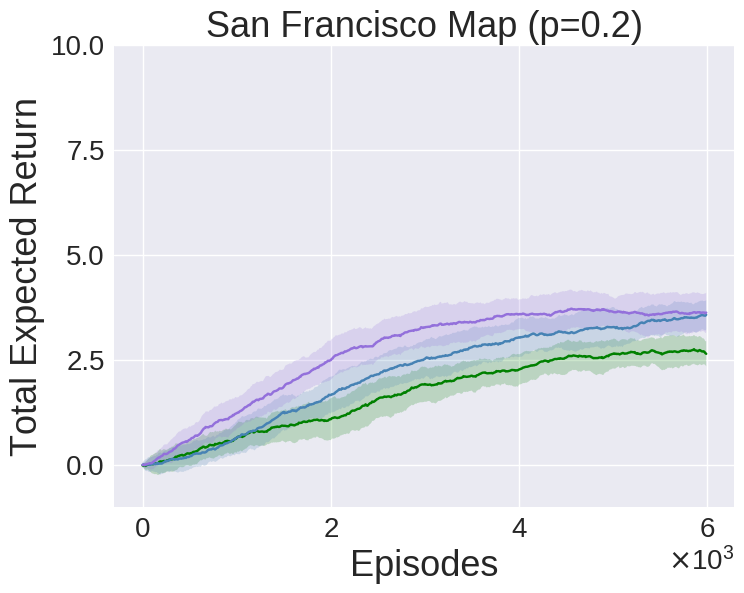}
		\includegraphics[width=0.32\textwidth]{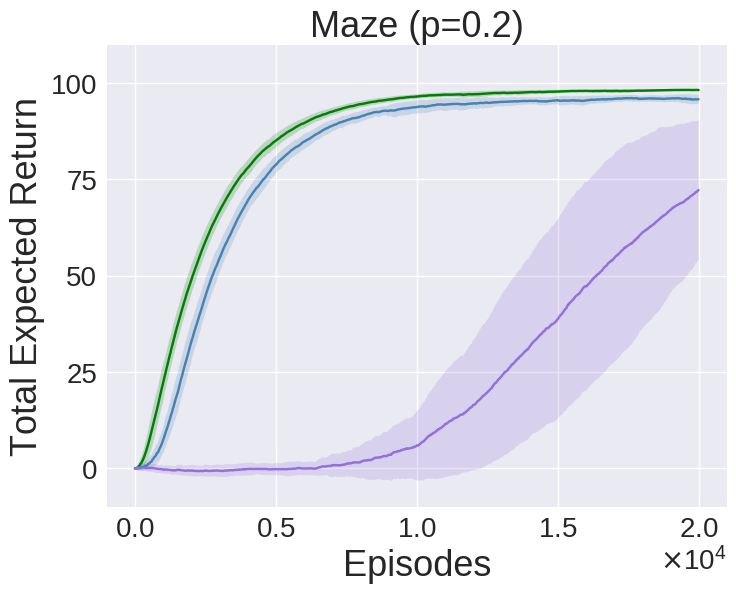}
		\includegraphics[width=0.32\textwidth]{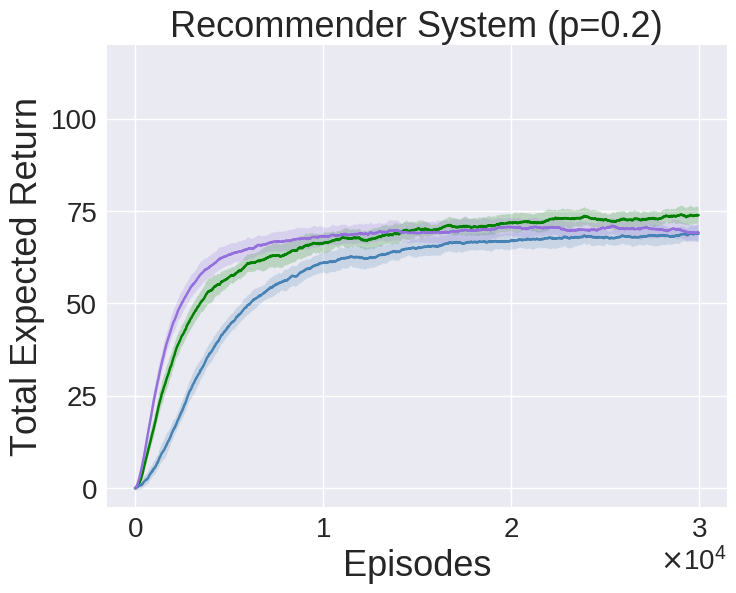}
		\\
		\includegraphics[width=0.6\textwidth]{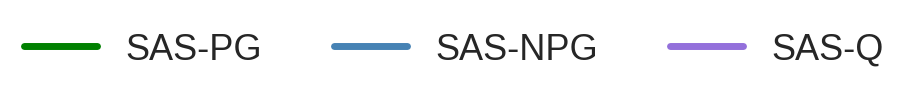}
		\caption{Best performing learning curves on different settings. (Left to right) San Francisco Map domain, Maze domain, and the Recommender System domain. 
		(Top to bottom) Probability of any action being available in the action set, ranging from $0.8$ to $0.2$, for the respective domains. 
		Shaded regions correspond to one standard deviation and were obtained using $30$ trials.}
		\label{apx:Fig:all_perf}
\end{figure}
\clearpage

\begin{figure}[t]
		\centering
		\includegraphics[width=0.32\textwidth]{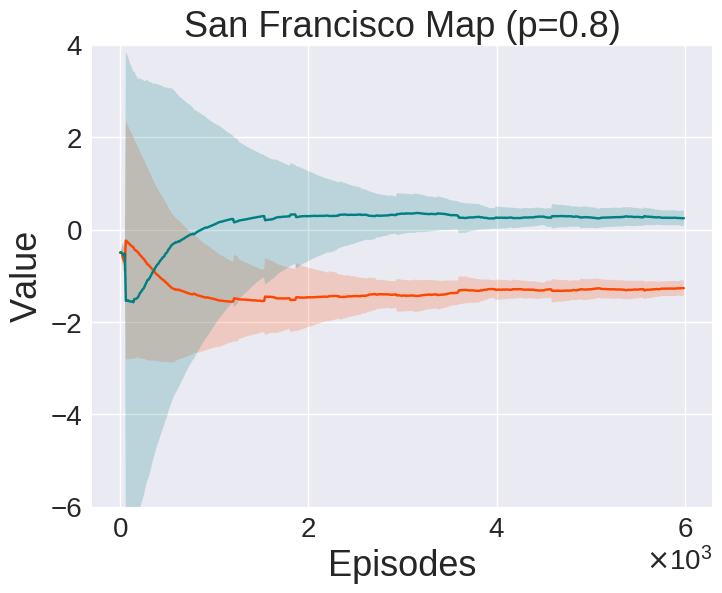}
		\includegraphics[width=0.32\textwidth]{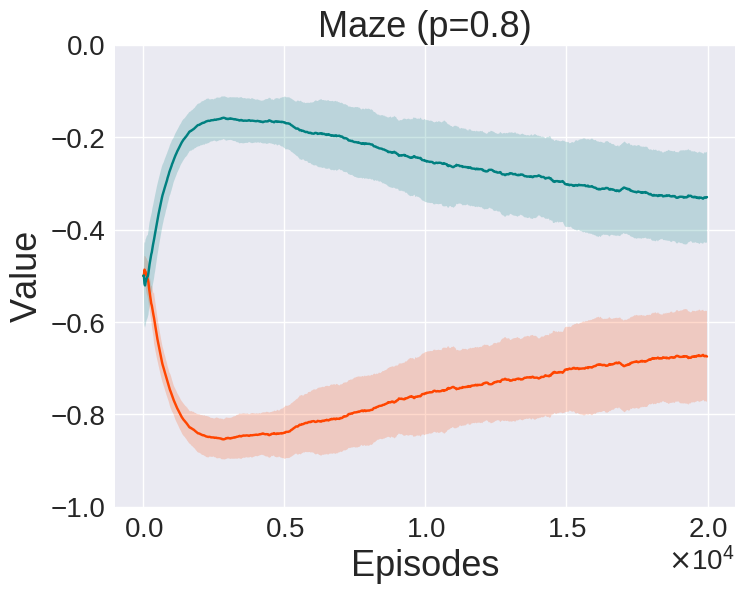}
		\includegraphics[width=0.32\textwidth]{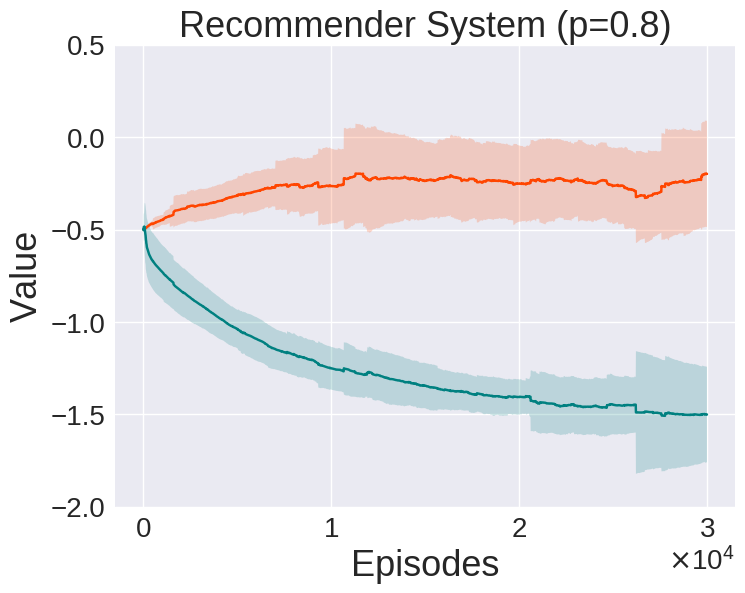}
		\\
		\includegraphics[width=0.32\textwidth]{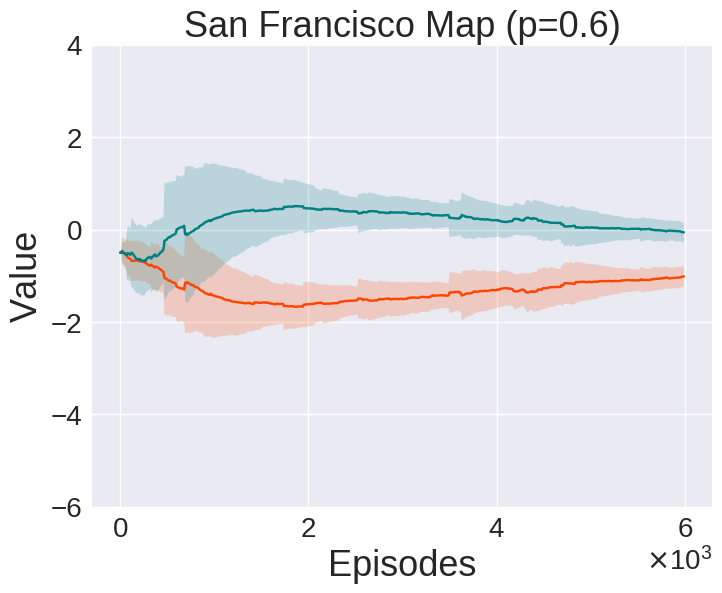}
		\includegraphics[width=0.32\textwidth]{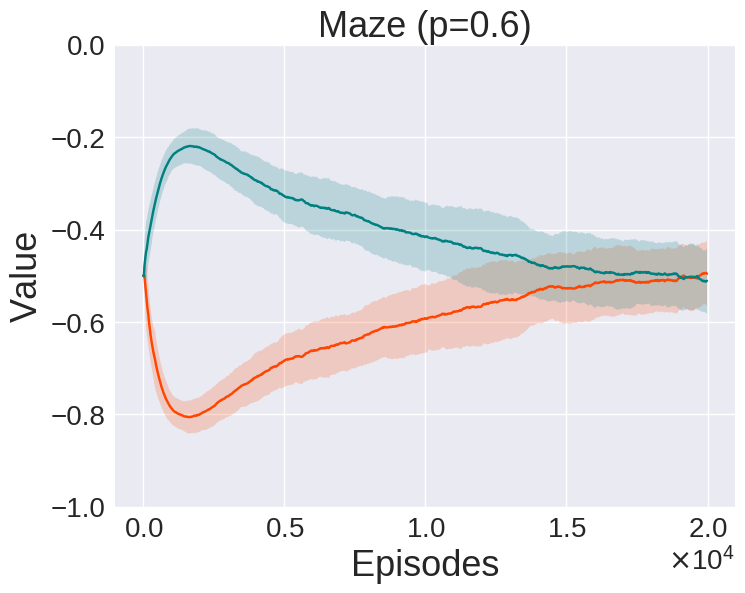}
		\includegraphics[width=0.32\textwidth]{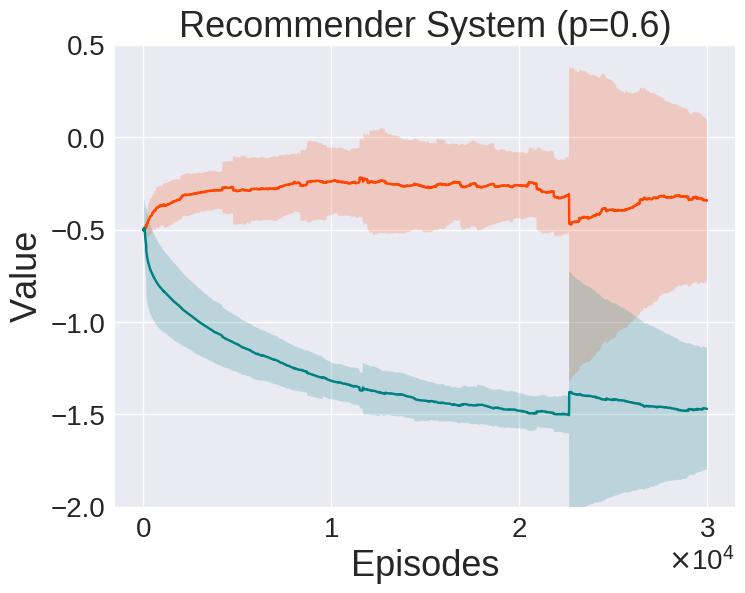}
		\\
		\includegraphics[width=0.32\textwidth]{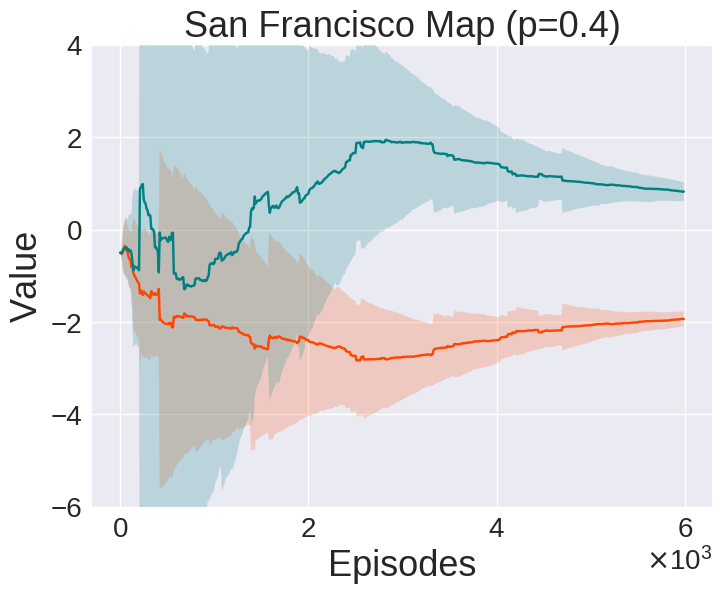}
		\includegraphics[width=0.32\textwidth]{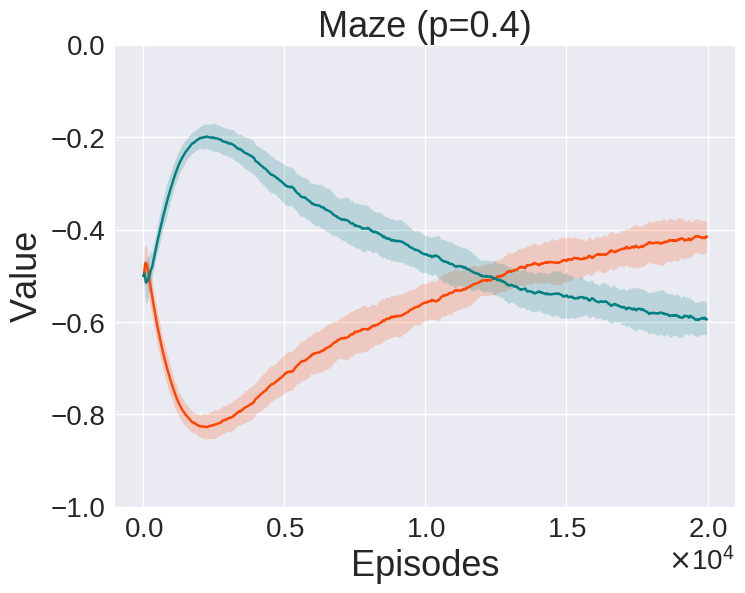}
		\includegraphics[width=0.32\textwidth]{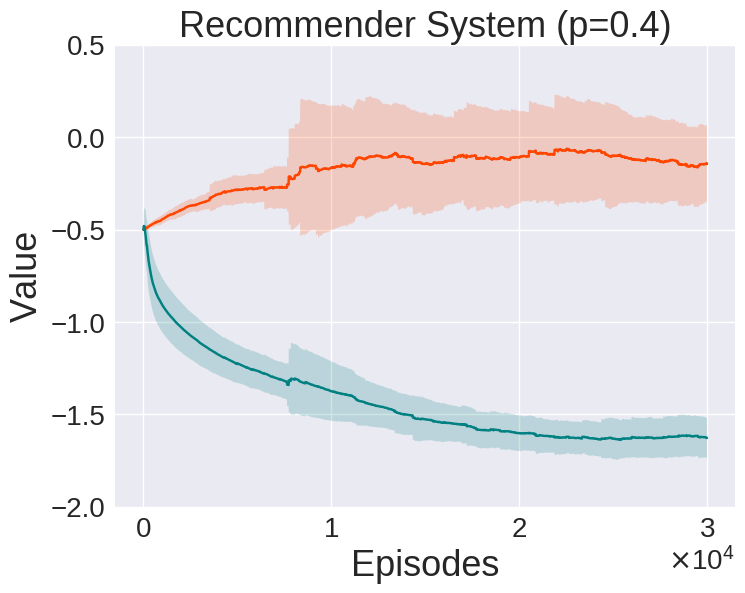}
		\\
		\includegraphics[width=0.32\textwidth]{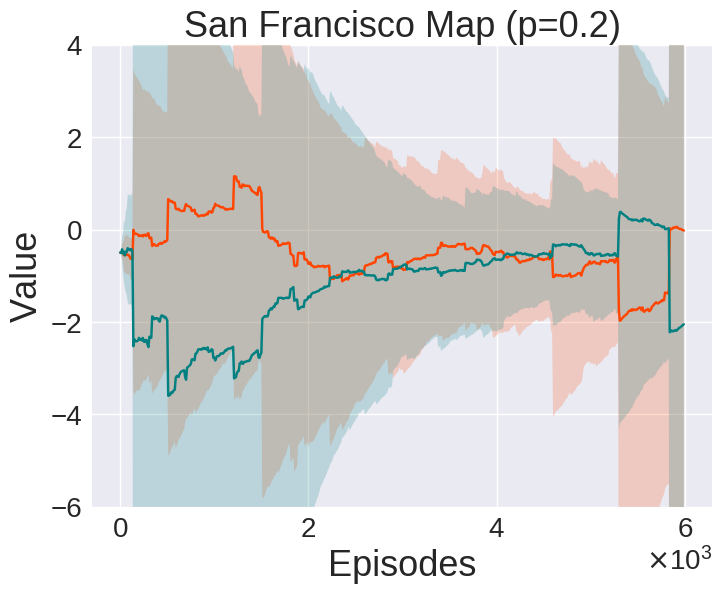}
		\includegraphics[width=0.32\textwidth]{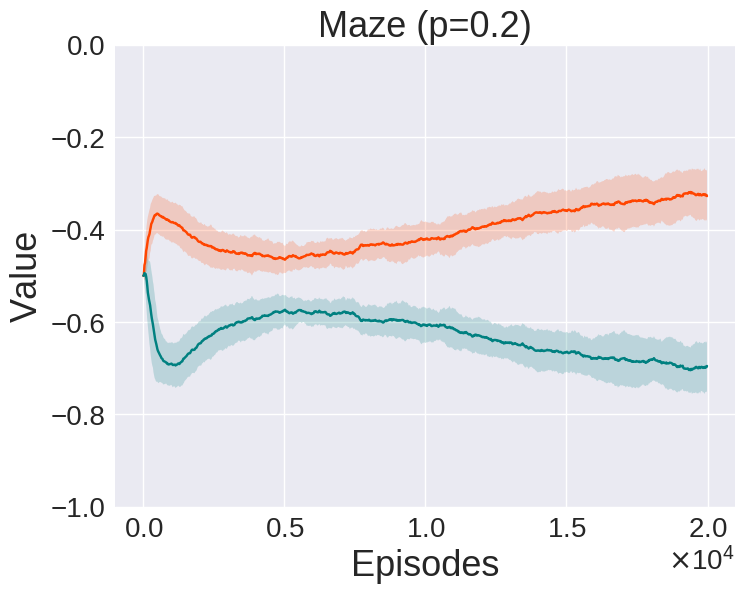}
		\includegraphics[width=0.32\textwidth]{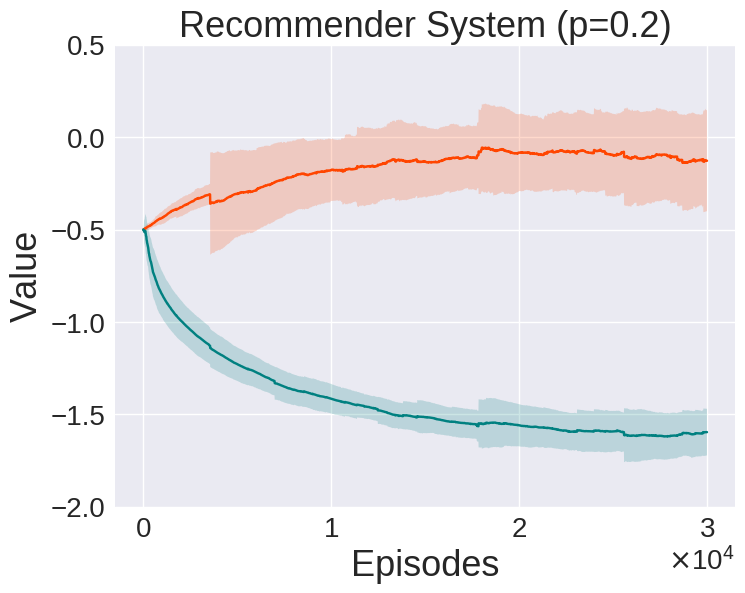}
		\\
		\includegraphics[width=0.6\textwidth]{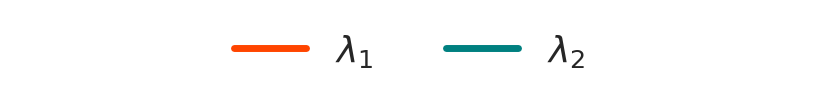}
		\caption{
		Autonomously adapted values of $\lambda_1$ and $\lambda_2$ (associated with $\hat v$ and $\bar q$, respectively) for the best performing SAS-PG instance on different settings. 
		(Left to right) San Francisco Map domain, Maze domain, and the Recommender System domain. 
		(Top to bottom) Probability of any action being available in the action set, ranging from $0.8$ to $0.2$, for the respective domains.
		Shaded regions correspond to one standard deviation and were obtained using $30$ trials.}
		\label{apx:Fig:all_prob}
\end{figure}
\clearpage

\end{document}